\documentclass[10pt]{article}
\pdfoutput=1
\usepackage{amssymb,amsmath,amsthm,amsfonts}
\usepackage{algorithm}
\usepackage[noend]{algpseudocode}
\usepackage{thmtools}
\usepackage{mathrsfs}
\usepackage[usenames,dvipsnames]{xcolor}
\usepackage{tikz}
\usepackage{hyperref}
\usepackage{bm}
\usepackage{booktabs}

\usepackage{graphicx}
\usepackage{cleveref}
\usepackage{fullpage}

\newif\ifhideproofs
\ifhideproofs
\usepackage{environ}
\NewEnviron{hide}{}

\fi

\newif\ifdraft
\drafttrue

\numberwithin{equation}{section}
\newtheorem{theorem}{Theorem}[section]
\newtheorem{corollary}[theorem]{Corollary}
\newtheorem{definition}[theorem]{Definition}
\newtheorem{lemma}[theorem]{Lemma}
\newtheorem{proposition}[theorem]{Proposition}

\newtheorem{remark}[theorem]{Remark}

\newcommand{\cF}{\mathcal{F}}

\newcommand{\cK}{\mathcal{K}}

\newcommand{\cN}{\mathcal{N}}

\newcommand{\cX}{\mathcal{X}}

\newcommand{\R}{\mathbb{R}}

\newcommand*{\ep}{\varepsilon}
\newcommand*{\eps}{\varepsilon}
\newcommand*{\del}{\delta}
\newcommand*{\sig}{\sigma}
\newcommand\mmid{\mathbin{\|}}

\providecommand{\norm}[1]{\lVert{#1}\rVert}

\newcommand{\ngd}{\texttt{NoisyGD}}
\newcommand{\nmgd}{\texttt{NoisyCGD}}
\newcommand{\ncgd}{\nmgd}
\newcommand{\nsgd}{\texttt{NoisySGD}}
\newcommand{\cni}{\text{CNI}}
\newcommand{\lmc}{\texttt{LMC}}

\def\Snospace~{\S{}}
\makeatletter
\renewcommand{\sectionautorefname}{\S\@gobble}
\renewcommand{\subsectionautorefname}{\S\@gobble}

\renewcommand{\appendixautorefname}{\S\@gobble}
\makeatother

\hypersetup{
    colorlinks = true,
    citecolor = Green,
    linkcolor = Black
}

\title{Shifted Interpolation for Differential Privacy}
\author{
	Jinho Bok \\
	UPenn\\
	\href{mailto:jinhobok@upenn.edu}{\texttt{\color{black}{jinhobok@upenn.edu}}}
	\and
	Weijie Su \\
	UPenn \\
	\href{mailto:suw@upenn.edu}{\texttt{\color{black}{suw@upenn.edu}}}
    \and
	Jason M. Altschuler\\
	UPenn\\
    \href{mailto:alts@upenn.edu}{\texttt{\color{black}{alts@upenn.edu}}}
}

\date{June 12, 2024}
\begin{document}

\maketitle

\begin{abstract}
    Noisy gradient descent and its variants are the predominant algorithms for differentially private machine learning. It is a fundamental question to quantify their privacy leakage, yet tight characterizations remain open even in the foundational setting of convex losses. This paper improves over previous analyses by establishing (and refining) the ``privacy amplification by iteration'' phenomenon in the unifying framework of $f$-differential privacy---which tightly captures all aspects of the privacy loss and immediately implies tighter privacy accounting in other notions of differential privacy, e.g., $(\eps,\delta)$-DP and R\'enyi DP. Our key technical insight is the construction of \emph{shifted interpolated processes} that unravel the popular shifted-divergences argument, enabling generalizations beyond divergence-based relaxations of DP. Notably, this leads to the first \emph{exact} privacy analysis in the foundational setting of strongly convex optimization. Our techniques extend to many settings: convex/strongly convex, constrained/unconstrained, full/cyclic/stochastic batches, and all combinations thereof. As an immediate corollary, we recover the $f$-DP characterization of the exponential mechanism for strongly convex optimization in \cite{gll22}, and moreover extend this result to more general settings.
\end{abstract}

\newpage
\setcounter{tocdepth}{2}
\tableofcontents

\section{Introduction}\label{sec:intro}

Private optimization is the primary approach for private machine learning. The goal is to train good models while not leaking sensitive attributes of the training data. Differential privacy (DP) is the gold standard for measuring this information leakage~\cite{dmns06,dwork2014algorithmic}, and noisy gradient descent and its variants are the predominant algorithms for private optimization. It is therefore a central question to quantify the differential privacy of these algorithms---however, tight characterizations remain open, even in the seemingly simple setting of convex optimization. 

In words, DP measures how distinguishable the output of a (randomized) algorithm is when run on two adjacent datasets, i.e., two datasets that differ in only one individual record. There are several ways to measure distinguishability---leading to many relaxations of DP, e.g., \cite{bs16, mir17, drs22}. Different DP notions lead to different privacy analyses, and a long line of work has sought to prove sharp privacy bounds for noisy gradient descent and its variants~\cite{bst14, dpsgd-dl, fmtt18, cys21, ys22,at22,abt24}.

\par A common approach is to use the composition theorem, which pays a price in privacy for every intermediate iterate along the optimization trajectory, leading to possibly suboptimal privacy bounds. Recent work has significantly improved the privacy analysis in the case of convex and strongly convex losses by showing that the privacy leakage of noisy (stochastic) gradient descent does not increase ad infinitum in the number of iterations $t$~\cite{cys21,ys22,at22,abt24}. This is in stark contrast to the composition-based approach, which gives privacy bounds that scale as $\sqrt{t}$.

All these ``convergent'' privacy bounds were proved in the R\'enyi DP framework, which is inherently lossy. To achieve the tightest possible privacy bound, a natural goal is to use the $f$-DP framework~\cite{drs22} for analysis, since it is an information-theoretically lossless definition of DP. This definition measures distinguishability in terms of the Type I vs Type II error tradeoff curve $f$ for the hypothesis testing problem of whether a given user was in the training dataset. The $f$-DP framework is desirable because: (1) $f$-DP exactly characterizes all relevant aspects of the hypothesis testing problem defining DP, and thus (optimal) $f$-DP bounds can be losslessly converted to (optimal) bounds in other notions of privacy such as $(\eps,\delta)$-DP or R\'enyi DP, (2) $f$-DP  is lossless under the composition of multiple private mechanisms, which is the most ubiquitous operation in DP since it enables combining building blocks, and (3) $f$-DP is easily interpretable in terms of the original hypothesis testing definition of DP. 

However, analyzing privacy leakage in the $f$-DP framework is often challenging since quantifying the entire tradeoff between Type I/II error is substantially more difficult than quantifying (less informative) alternative notions of privacy. Consequently, the analysis toolbox for $f$-DP is currently limited. These limitations are pronounced for the fundamental problem of analyzing the privacy loss of noisy gradient descent and its variants. 
To put this into perspective, existing privacy guarantees based on $f$-DP \emph{diverge} as the number of iterations $t$ increases, whereas the aforementioned recent work has used divergence-based DP definitions to show that noisy gradient descent and its variants can remain private even when run indefinitely, for problems that are strongly convex~\cite{cys21,ys22,at22,abt24} or even just convex~\cite{at22,abt24}. 
Convergent privacy bounds complement celebrated results for minimax-optimal privacy-utility tradeoffs~\cite{bst14, bftg19} because they enable longer training---which is useful since typical learning problems are not worst-case and benefit from training longer. 

\par Can convergent privacy bounds be achieved directly\footnote{Convergent RDP bounds can of course be \emph{lossily} converted to convergent $f$-DP bounds, but that defeats the purpose of using the lossless $f$-DP framework.} in the tight framework of $f$-DP? All current arguments are tailored to R\'enyi DP---an analytically convenient but inherently lossy relaxation of DP---and do not appear to extend. Answering this question necessitates developing fundamentally different techniques for $f$-DP, since convergent privacy bounds require only releasing the algorithm's \emph{final} iterate---in sharp contrast to existing $f$-DP techniques such as the composition theorem which can only argue about the accumulated privacy loss of releasing \emph{all} intermediate iterates. Tight $f$-DP analyses typically require closed-form expressions for the random variable in question---in order to argue about the tradeoff of Type I/II error---but this is impossible for the final iterate of (stochastic) gradient descent due to the non-linearity intrinsic to each iteration.

\subsection{Contribution}\label{ssec:cont}

Our primary technical contribution is establishing (and refining) the ``privacy amplification by iteration'' phenomenon in the unifying framework of $f$-DP. This enables directly analyzing the privacy loss of the final iterate of noisy gradient descent (and its variants), leading to the first direct $f$-DP analysis that is convergent as the number of iterations $t \to \infty$.~\autoref{ssec:tech} overviews this new analysis technique.

Notably, this yields the first \emph{exact} privacy analysis in the foundational setting of strongly convex losses. To our knowledge, there is no other setting where exact privacy analyses are known for any $t > 1$, except for the setting of convex quadratic losses which is analytically trivial because all iterates are explicit Gaussians.\footnote{The standard analysis approach based on the composition theorem is nearly tight for small numbers of iterations $t$, but as mentioned above, yields an arbitrarily loose bound (in fact vacuous) for convex losses as $t \to \infty$.} 

\par We emphasize that our techniques are versatile and readily extend to many settings---a well-known challenge for other convergent analyses, even for simpler relaxations of DP like R\'enyi DP~\cite{cys21, ys22,at22,abt24}. In \autoref{sec:bounds}, we illustrate how our analysis extends to convex/strongly convex losses, constrained/unconstrained optimization, full/cyclic/stochastic batches, and all combinations thereof.

\par Since our improved privacy guarantees are for $f$-DP (\autoref{fig:fdp-opt}, left), lossless conversions immediately imply improved guarantees for other notions of privacy like R\'enyi DP and $(\eps,\delta)$-DP (\autoref{fig:fdp-opt}, right). For example, for the strongly convex setting, our exact bound improves over previous results by a factor of $2$ in R\'enyi DP, and thus by even more in $(\eps,\delta)$-DP due to the intrinsic lossiness of R\'enyi DP that we overcome by directly analyzing in $f$-DP. In practice, improving the privacy by a factor of two enables training with half the noise, while satisfying the same privacy budget. 
Although this paper's focus is the theoretical methodology, preliminary numerics in \autoref{sec:example} corroborate that our improved privacy guarantees can be helpful in practice.

\begin{figure}
    \centering
    \includegraphics[width=0.7\textwidth]{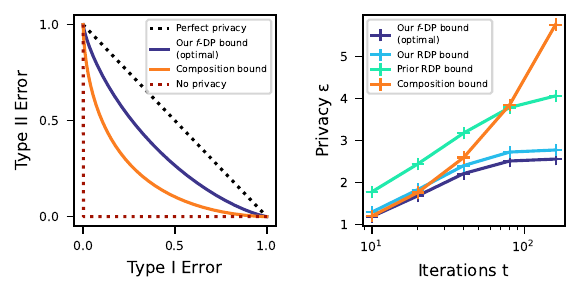}
     \caption{
     Left: improved $f$-DP versus the standard composition analysis. Right: improved $(\eps,\delta)$-DP by losslessly converting from $f$-DP. 
     Our privacy bound is optimal in all parameters, here for $\ngd$ on strongly convex losses; see \autoref{app:num} for the parameter choices and other settings. 
     Our $f$-DP analysis also implies optimal bounds for the R\'enyi DP framework (previously unknown), but $f$-DP is strictly better since it captures all aspects of the privacy leakage, whereas R\'enyi DP is intrinsically lossy.
 } 
    \label{fig:fdp-opt}
\end{figure}

\par Since our privacy bounds are convergent in the number of iterations $t$, we can take the limit $t \to \infty$ to bound the $f$-DP of the stationary distributions of these optimization algorithms. As an immediate corollary, we recover the recent $f$-DP characterization of the exponential mechanism for strongly convex losses in \cite{gll22}, and moreover extend this result to more general settings in \autoref{sec:expmech}.

\subsection{Techniques}\label{ssec:tech}

The core innovation underlying our results is the construction of certain auxiliary processes, \emph{shifted interpolated processes}, which enable directly analyzing the Type I/II error tradeoff between the final iterates of two stochastic processes---even when their laws are complicated and non-explicit. Informally, this argument enables running \emph{coupling arguments}---traditionally possible only for Wasserstein analysis---to analyze tradeoff functions for the first time. In this paper, the two processes are noisy (stochastic, projected) gradient descent run on two adjacent datasets, but the technique is more general and we believe may be of independent interest. See~\autoref{sec:tech} for a detailed overview.

\par Crucially, our argument is \emph{geometrically aware}: it exploits (strong) convexity of losses via (strong) contractivity of gradient descent updates, in order to argue that sensitive gradient queries have (exponentially) decaying privacy leakage, the longer ago they were performed. This is essential for convergent privacy bounds, and is impossible with the standard composition-based analysis---which only exploits the sensitivity of the losses, and is oblivious to any further geometric phenomena like convexity or contractivity. 

\par A key motivation behind the construction of our auxiliary sequence is that it demystifies the popular privacy amplification by iteration analysis \cite{fmtt18}, which has been used in many contexts, and in particular was recently shown to give convergent R\'enyi DP bounds~\cite{at22,abt24}. Those arguments rely on \emph{shifted divergences}, which combine R\'enyi divergence and Wasserstein distance, and it was an open question whether this ad-hoc potential function could be simplified. Our shifted interpolated process answers this: its iterates coincide with the optimal ``shifts'' in the shifted divergence argument, which allows us to disentangle the R\'enyi and Wasserstein components of the shifted divergence argument; details in~\autoref{sec:tech:disentangle}. Crucially, this disentanglement enables generalizations 
beyond divergence-based relaxations of DP, to $f$-DP.\footnote{Na\"ively extending the ``shifted divergence'' argument
to ``shifted tradeoffs'' runs into subtle but fundamental issues since tradeoff functions do not enjoy key properties that divergences do. Details in~\autoref{sec:tech:disentangle}.
}

\subsection{Outline}

\autoref{sec:prelim} recalls relevant preliminaries from differential privacy and convex optimization. \autoref{sec:tech} introduces the technique of shifted interpolation. \autoref{sec:bounds} uses this to establish improved privacy bounds for noisy gradient descent and its variants in a number of settings. \autoref{sec:expmech} describes how, as immediate corollaries of these convergent privacy bounds, taking an appropriate limit recovers and generalizes recent results on the $f$-DP of the exponential mechanism. \autoref{sec:disc} discusses future directions motivated by our results. For brevity, various helper lemmas, proof details, and additional numerical experiments are deferred to the Appendix. Code reproducing our numerics can be found here: \url{https://github.com/jinhobok/shifted_interpolation_dp}.

\section{Preliminaries}\label{sec:prelim}

\subsection{Differential privacy}

DP measures the distinguishability between outputs of a randomized algorithm run on adjacent datasets, i.e., datasets that differ on at most one data point~\cite{dmns06}. The most popular definition is $(\ep, \delta)$-DP.
\begin{definition}[$(\eps,\delta)$-DP]
    A randomized algorithm $\mathcal{A}$ is \emph{$(\ep, \delta)$-DP} if for any adjacent datasets $S, S'$ and any event $E$,
    \[\mathbb{P}(\mathcal{A}(S) \in E) \leq e^\ep \mathbb{P}(\mathcal{A}(S') \in E) + \delta\,.
    \]
\end{definition}

However, the most precise quantification of DP is based on the hypothesis-testing formulation~\cite{wz10, kov17}.  This is formalized as $f$-DP~\cite{drs22}, where $f$ denotes a tradeoff function, i.e., a curve of hypothesis testing errors for a hypothesis test $\phi$.

\begin{definition}[$f$-DP]
For distributions $P,Q$ on the same space, 
the \emph{tradeoff function} $T(P, Q): [0, 1] \to [0, 1]$ is
\[T(P, Q)(\alpha) = \inf\{1 - \mathbb{E}_Q \phi: \mathbb{E}_P \phi \leq \alpha, 0 \leq \phi \leq 1\}\,.
\]
 A randomized algorithm $\mathcal{A}$ is \emph{$f$-DP} if for any adjacent datasets $S$ and $S'$, $T(\mathcal{A}(S), \mathcal{A}(S')) \geq f$.
\end{definition}

Here and henceforth we use the pointwise ordering between tradeoff functions, i.e., we write $f \geq g$ if $f(\alpha) \geq g(\alpha)$ for all $\alpha$ in the domain $[0,1]$. We also use the standard abuse of notation of writing $T(X,Y)$ as shorthand for $T(\mathrm{Law}(X), \mathrm{Law}(Y))$.

The following lemma provides a useful characterization of tradeoff functions \cite[Proposition 1]{drs22}. It follows that the most private tradeoff function is $\text{Id}: [0, 1] \to [0, 1]$, given by $\text{Id}(\alpha) = 1-\alpha$. See \autoref{fig:fdp-intro}.

\begin{lemma}[Characterization of tradeoff functions]\label{prop:fdpchar}
A function $f: [0, 1] \to [0, 1]$ is a tradeoff function iff $f$ is decreasing, convex and $f(\alpha) \leq 1 - \alpha$ for all $\alpha \in [0, 1]$.    
\end{lemma}

 See \autoref{subsec:tradeoff} for further details on tradeoff functions. Gaussian tradeoff functions are a particularly useful family, providing a notion of Gaussian DP (GDP) parametrized by a single scalar. These are central to our analysis due to the Gaussian noise in noisy (stochastic) gradient descent.

\begin{definition}[GDP]
For $\mu \geq 0$, the \emph{Gaussian tradeoff function} $G(\mu)$ is defined as $G(\mu) = T(\cN(0, 1), \cN(\mu, 1))$. Its value at $\alpha \in [0, 1]$ is given as $G(\mu)(\alpha) = \Phi(\Phi^{-1}(1-\alpha) - \mu)$, where $\Phi$ denotes the CDF of $\cN(0, 1)$. A randomized algorithm $\mathcal{A}$ is \emph{$\mu$-GDP} if for any adjacent datasets $S$ and $S'$, $T(\mathcal{A}(S), \mathcal{A}(S')) \geq G(\mu)$.
\end{definition}

\begin{figure}
    \centering
    \includegraphics[width=0.5\textwidth]{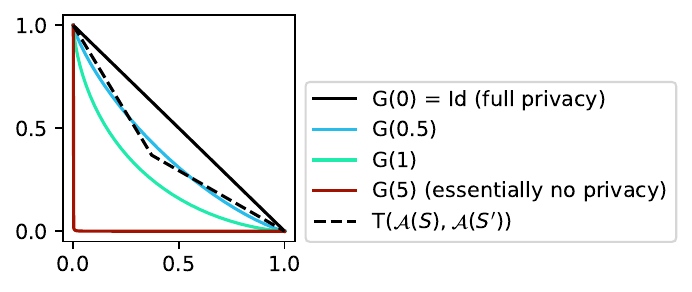}
    \caption{Illustration of $f$-DP and GDP. Gaussian tradeoff functions $G(\mu)$ are less private as $\mu$ increases from $0$ (full privacy) to $\infty$ (no privacy). The closer to $\text{Id}$, the more private. Here $\mathcal{A}$ is $1$-GDP but not $0.5$-GDP because its tradeoff function is pointwise above $G(1)$ but not pointwise above $G(0.5)$.
    }
    \label{fig:fdp-intro}
\end{figure}

We now recall two key properties of tradeoff functions that are central to our analysis.
The first states that post-processing two distributions cannot make them easier to distinguish \cite[Lemma 1]{drs22}.

\begin{lemma}[Post-processing]\label{lem:proc} 
For any probability distributions $P, Q$ and (random) map $\text{Proc}$,
\[T(\text{Proc}(P), \text{Proc}(Q)) \geq T(P, Q)\,.
\]
\end{lemma}
The next lemma enables analyzing the composition of multiple private mechanisms~\cite[Definition 5 \& Lemma C.1]{drs22}. 

\begin{definition}[Composition]
The \emph{composition} of two tradeoff functions $f = T(P, Q)$ and $g = T(P', Q')$ is defined as 
$f \otimes g = T(P \times P', Q \times Q')$.
The $n$-fold composition of $f$ with itself is denoted as $f^{\otimes n}$.
\end{definition}

\begin{lemma}[Strong composition]\label{lem:strcomp} Let $K_1, K'_1,K_2, K'_2$ be (random) maps such that for all $y$, 
\[T(K_1(y), K'_1(y)) \geq T(K_2(y), K'_2(y))\,.
\] 
Then 
\[T((P, K_1(P)), (Q, K'_1(Q))) \geq T((P, K_2(P)), (Q, K'_2(Q)))\,.
\]
In particular, if $g = T(K_2(y), K'_2(y))$ does not depend on $y$, then $T((P, K_1(P)), (Q, K'_1(Q))) \geq T(P, Q) \otimes g$.
\end{lemma}

\subsection{Convex optimization}

This paper focuses on convex losses because tight privacy guarantees for noisy gradient descent (and variants) are open even in this seemingly simple setting. 
We make use of the following two basic facts from convex optimization. Below, we say a function is contractive if it is $1$-Lipschitz. Recall that a function $f$ is $M$-smooth if $\nabla f$ is $M$-Lipschitz, 
and is $m$-strongly convex if $x \mapsto f(x) - \frac{m}{2}\lVert x \rVert^2$ is convex.

\begin{lemma}\label{lem:contraction} If $f$ is convex and $M$-smooth, then the gradient descent update
$g(x) = x - \eta \nabla f(x)$
is contractive for each $\eta  \in [0, 2/M]$. If $f$ is additionally $m$-strongly convex and $\eta \in (0, 2/M)$, then $g$ is $c$-Lipschitz where $c = \max\{|1 - \eta m|, |1 - \eta M|\} < 1$.
\end{lemma}

\begin{lemma}\label{lem:projection} Let $\cK$ be a closed and convex set in $\R^d$. Then the projection $\Pi_{\cK}(x) = \arg \min_{z \in \cK} \norm{z - x}$ is well-defined and contractive.
\end{lemma}

\subsection{Private optimization algorithms}\label{ssec:prelim:algs}
Throughout, we consider a private optimization setting in which the goal is to minimize an objective function $F(x) = \frac{1}{n}\sum_{i=1}^n f_i(x)$, where the $i$-th loss function $f_i$ is associated with the $i$-th data point in a dataset $S$. An adjacent dataset $S'$ corresponds to loss functions $\{f'_i\}_{i \in [n]}$ where $f_i \equiv f'_i$ except for a single index $i^*$.

Noisy gradient descent and its variants follow the general template of
\begin{equation}\label{eq:alg}
    X_{k+1} \gets \Pi_{\cK} \left[ X_k - \eta \left( \frac{1}{b} \sum_{i \in B_k} \nabla f_i(X_k) + Z_{k+1} \right) \right], \qquad k = 0, 1, \dots, t-1
\end{equation}
where $X_0$ is the initialization (e.g., zero), $\eta$ is the learning rate, $Z_{k+1} \sim \cN(0,\sigma^2 I_d)$ independently, $\sigma$ is the noise rate, $\cK$ is the constraint set, and $t$ is the number of steps. The batch $B_k$ of size $b$ can be chosen in several ways:
\begin{itemize}
    \item Full batches ($\ngd$): $B_k \equiv [n]$.
    \item Cyclic batches ($\ncgd$): Partition $[n]$ into batches of sizes $b$ and cycle through them.   
    \item Stochastic batches ($\nsgd$): Choose batches of size $b$ uniformly at random from $[n]$.
\end{itemize}
The advantage of the latter two variants is that they avoid computing the gradient of the objective, which can be computationally burdensome when $n$ is large.

A standard assumption in private optimization is the following notion of gradient sensitivity:

\begin{definition}[Gradient sensitivity]
 A family of loss functions $\cF$ (defined on $\cX$) has \emph{gradient sensitivity} $L$ if 
 \[\sup_{f,g \in \cF, x \in \cX} \norm{\nabla f(x) - \nabla g(x)} \leq L\,.
 \]
\end{definition}

For example, a family of $L$-Lipschitz loss functions has gradient sensitivity $2L$.  Another example is loss functions of the form $f_i = \ell_i + r$, where $\ell_i$ are convex, $L$-Lipschitz losses, and $r$ is a (non-Lipschitz) strongly convex regularization---the point being that this family of loss functions $\{f_i\}$ has finite gradient sensitivity $2L$ despite each $f_i$ not being Lipschitz.

\section{Shifted interpolation for $f$-DP}\label{sec:tech}

Here we explain the key conceptual ideas enabling our convergent $f$-DP bounds (see~\autoref{ssec:tech} for a high-level discussion). To preserve the flow of ideas we defer proofs to \autoref{app:shifted}. Below, in~\autoref{ssec:tech:prev} we first recall the standard $f$-DP analysis based on the composition theorem and why it yields divergent bounds. Then in~\autoref{ssec:tech:new} we describe our technique of shifted interpolated processes and how this enables convergent $f$-DP bounds. 

To explain the ideas in their simplest form, we consider here the setting of full-batch gradients and unconstrained optimization. Let $\{f_i\}_{i \in [n]}$ and $\{f_i'\}_{i \in [n]}$ be the losses corresponding to two adjacent datasets, where $f_i \equiv f_i'$ except for one index $i^*$. Then $\ngd$ forms the iterates
\begin{align}
    X_{k+1} &= \phi(X_k) + Z_{k+1} \label{eq:tech-X} \\
    X_{k+1}' &= \phi'(X_k') + Z'_{k+1} \label{eq:tech-X'} 
\end{align}
where $X_0 = X_0'$, $\phi(x) := x - \tfrac{\eta}{n}\sum_{i=1}^n \nabla f_i(x)$, $\phi'(x') := x' - \tfrac{\eta}{n} \sum_{i=1}^n \nabla f_i'(x')$, and $Z_{k+1}, Z'_{k+1} \sim \cN(0, \eta^2 \sig^2 I_d)$. 

\subsection{Previous (divergent) $f$-DP bounds, via composition}\label{ssec:tech:prev}

$f$-DP requires bounding $T(X_t, X_t')$. The standard approach, based on the composition theorem, argues as follows:
\begin{align}
    T(X_t, X_t') &\geq T(X_{t-1}, X_{t-1}') \otimes G(c) \nonumber \\
    &\geq T(X_{t-2}, X_{t-2}') \otimes G(c\sqrt{2}) \nonumber \\
    & \cdots \nonumber \\
    & \geq \underbrace{T(X_0,X_0')}_{= \text{Id since } X_0 = X_0'} \otimes \; G(c\sqrt{t})
    \,. \label{eq:tech:prev}
\end{align}
Here, the composition theorem simultaneously ``unrolls'' both processes, at some price $G(c)$ in each iteration. (These prices are collected via a basic GDP identity, \autoref{lem:various}.) This is due to the following simple lemma, which relies on the $f$-DP of the Gaussian mechanism using different updates $\phi,\phi'$~\cite[Theorem 2]{drs22}. 

\begin{lemma}\label{lem:cni} 
Suppose $\norm{\phi(x) - \phi'(x)} \leq s$ for all $x$. Then
\[T(\phi(X) + \cN(0, \sigma^2 I_d), \phi'(X') + \cN(0, \sigma^2 I_d)) \geq T(X, X') \otimes G(\frac{s}{\sigma})\,.
\]
\end{lemma}
Bounding $s$ via sensitivity enables the argument~\eqref{eq:tech:prev} and gives the appropriate $c$.
See \cite{drs22} for details. However, while this argument~\eqref{eq:tech:prev} is reasonably tight for small $t$, it is vacuous as $t \to \infty$. Conceptually, this is because this analysis considers releasing all intermediate iterates, hence it bounds $T((X_1, \dots, X_t), (X_1', \dots X_t')) \geq G(c\sqrt{t})$. Concretely, this is because the above analysis requires \emph{completely} unrolling to iteration $0$. Indeed, the identical initialization $X_0 = X_0'$ ensures $T(X_0, X_0') = \text{Id}$, whereas at any other iteration $k > 0$ it is unclear how to directly bound $T(X_k,X_k')$ as $X_k \neq X_k'$. This inevitably leads to final privacy bounds which \emph{diverge} in $t$ since a penalty is incurred in each of the $t$ iterations. 

\subsection{Convergent $f$-DP bounds, via shifted interpolation}\label{ssec:tech:new}

The central idea underlying our analysis is the construction of a certain \emph{auxiliary process} $\{\widetilde{X}_k\}$ that interpolates between the two processes in the sense that $\widetilde{X}_{\tau} = X_{\tau}'$ at some intermediate time $\tau$ and $\widetilde{X}_{t} = X_t$ at the final time. See \autoref{fig:shift-int}. Crucially, this enables running the argument~\eqref{eq:tech:new} where we unroll only from $t$ to $\tau$, rather than all the way to initialization:
\begin{align}
    T(X_t, X_t') &= T(\widetilde{X}_t, X_t') \nonumber \\
    &\geq T(\widetilde{X}_{t-1}, X_{t-1}') \otimes G(a_{t}) \nonumber \\
    &\geq T(\widetilde{X}_{t-2}, X_{t-2}') \otimes 
    G\big(\big(a_{t}^2 + a_{t-1}^2\big)^{1/2} \big) 
    \nonumber \\
    & \dots \nonumber \\
    & \geq \underbrace{T(\widetilde{X}_{\tau},X_{\tau}')}_{= \text{Id since } \tilde{X}_{\tau} = X_{\tau}'} \otimes \; G\big(\big(\sum_{k=\tau+1}^{t} a_k^2\big)^{1/2}\big) 
    \,. \label{eq:tech:new}
\end{align}
Intuitively, this argument replaces the divergent $\sqrt{t}$ dependence of prior $f$-DP bounds with something scaling in $t - \tau$. Here $\tau$ is an analysis parameter that we can optimize based on the following intuitive tradeoff: larger $\tau$ enables unrolling less, whereas smaller $\tau$ gives the auxiliary process $\{\widetilde{X}_k\}$ more time to interpolate between $X_{\tau}'$ and $X_t$ which leads to smaller penalties $a_k$ for unrolling at each iteration.

\begin{figure}
    \centering
    \includegraphics[width=0.5\textwidth]{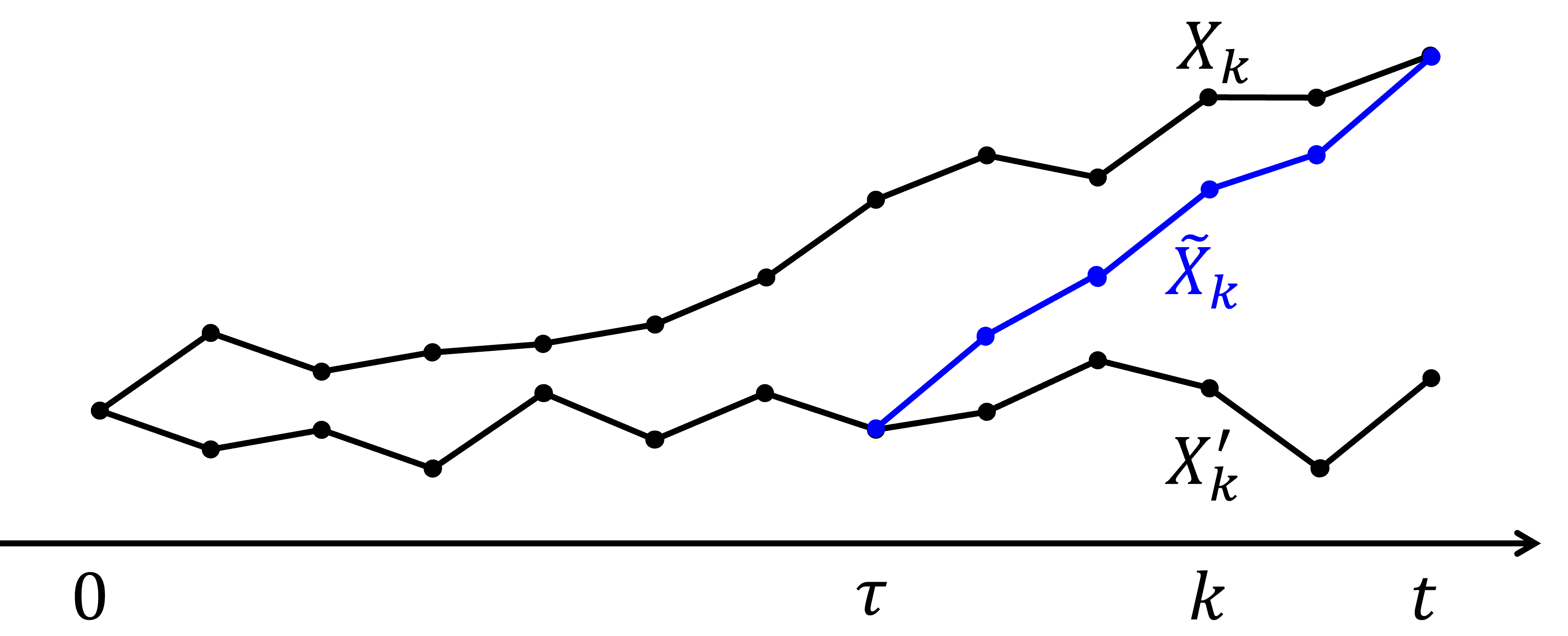}
    \caption{Illustration of the shifted interpolated process $\{\widetilde{X}_k\}$ defined in \eqref{eq:shift-simple}. It starts from one process ($\widetilde{X}_\tau = X'_\tau$) and ends at the other ($\widetilde{X}_t = X_t$). The intermediate time $\tau$ is an analysis parameter that we optimize to get the best final privacy bound.}
    \label{fig:shift-int}
\end{figure}

Formalizing~\eqref{eq:tech:new} leads to two interconnected questions:
\begin{itemize}
    \item \textbf{Q1.} How to construct the auxiliary process $\{\widetilde{X}_k\}$? 
    \item \textbf{Q2.} How to unroll each iteration? I.e., what is the analog of \autoref{lem:cni}?
\end{itemize}

\subsubsection{Shifted interpolated process} For Q1, we initialize $\widetilde{X}_{\tau} = X_{\tau}'$ and define 
\begin{equation}
\begin{aligned}\label{eq:shift-simple}
    \widetilde{X}_{k+1} = \lambda_{k+1}\phi(X_k) + (1 - \lambda_{k+1})\phi'(\widetilde{X}_k) + Z_{k+1}
\end{aligned}
\end{equation}
for $k = \tau, \dots, t-1$. Intuitively, this auxiliary process $\{\widetilde{X}_k\}$ uses a convex combination of the updates performed by the two processes $\{X_k\}$ and $\{X_k'\}$, enabling it to gracefully interpolate from its initialization at one process to its termination at the other. Here $\lambda_k$ controls the speed at which we \emph{shift} from one process to the other. We set $\lambda_{t} = 1$ so that $\widetilde{X}_t = X_t$ achieves the desired interpolation; the other $\{\lambda_{k}\}$ are analysis parameters that we optimize to get the best final bound. An important technical remark is that this auxiliary process uses the same noise increments $\{Z_k\}$ as $\{X_k\}$; this coupling enables bounding the distance between $X_k$ and $\widetilde{X}_k$ by a deterministic value (i.e., in the $\infty$-Wasserstein distance $W_{\infty}$). 

We remark that auxiliary interpolating processes have been used in the context of proving Harnack inequalities (or equivalently, R\'enyi reverse transport inequalities) for diffusions on manifolds~\cite{arnaudon2006harnack, wang2013harnack, wang2014analysis,shiftcomp1,shiftcomp2}. Two key challenges posed by the present setting are that $f$-DP requires tradeoff functions (rather than R\'enyi divergences), and also tracking stochastic processes that undergo \emph{different} dynamics (rather than the same diffusion). This requires constructing and analyzing the auxiliary process~\eqref{eq:shift-simple}.

\subsubsection{Geometrically aware composition} 

For Q2, we develop the following lemma, which generalizes \autoref{lem:cni} by allowing for an auxiliary process $\widetilde{X}$ and a shift parameter $\lambda$ (\autoref{lem:cni} is recovered in the special case $\lambda = 1$ and $\widetilde{X} = X$). A key feature is that unlike \autoref{lem:cni}, this lemma is \emph{geometrically aware} in that it exploits the Lipschitzness of the gradient descent updates $\phi,\phi'$---recall from \autoref{lem:contraction} that $\phi,\phi'$ are (strongly) contractive whenever the losses are (strongly) convex. Intuitively, this contractivity ensures that long-ago gradient queries incur (exponentially) less privacy loss, thus making the total privacy loss convergent; c.f., the discussion in~\autoref{ssec:tech}. 

\begin{lemma}\label{lem:cni-shift} Suppose that $\norm{\phi(x) - \phi'(x)} \leq s$ for all $x$ and that $\phi, \phi'$ are $c$-Lipschitz. Then for any $\lambda \geq 0$ and any random variable $\widetilde{X}$ satisfying $\norm{X - \widetilde{X}} \leq z$, 
\[T(\lambda \phi(X) + (1-\lambda) \phi'(\widetilde{X}) + \cN(0, \sigma^2 I_d), \phi'(X') + \cN(0, \sigma^2 I_d)) \geq T(\widetilde{X}, X') \otimes G(\frac{\lambda(cz + s)}{\sigma})\,.
\]
\end{lemma}

\subsubsection{Convergent $f$-DP bounds}

Combining our answers to Q1 (shifted interpolated process) and Q2 (geometrically aware composition) enables formalizing the argument~\eqref{eq:tech:new}. The remaining proof details are straightforward and deferred to \autoref{app:noisygd}. For clarity, we state this result as a ``meta-theorem'' where the shifts $\lambda_k$ and intermediate time $\tau$ are parameters; our final bounds are obtained by optimizing them, see \autoref{sec:bounds}.

\begin{theorem}\label{thm:shift-main} 
    Consider the stochastic processes $\{X_k\}, \{X_k'\}, \{\widetilde{X}_k\}$ defined in~\eqref{eq:tech-X},~\eqref{eq:tech-X'},~\eqref{eq:shift-simple}, with $\lambda_t = 1$. Suppose that $\norm{\phi(x) - \phi'(x)} \leq s$ for all $x$ and that $\phi, \phi'$ are $c$-Lipschitz. For any sequence $\{z_k\}$ such that $\norm{X_k - \widetilde{X}_k} \leq z_k$,
\[
T(X_t, X'_t) \geq G\left(\frac{1}{\sigma}\sqrt{\sum_{k=\tau+1}^{t} a_k^2} \right)
\]
where $a_k = \lambda_k(cz_{k-1} + s)$.
\end{theorem}

We emphasize that although this technique-overview section focused on the simple case of full-batch gradients and strongly convex losses for clarity, these techniques readily extend to more general settings. Briefly, for constrained optimization, projections are handled by using the post-processing inequality for tradeoff functions; for (non-strongly) convex optimization, the optimal shifts $a_k$ will be of similar size rather than geometrically increasing (and $\tau$ will be strictly positive rather than zero); for cyclic batches, the update functions $\phi_k,\phi_k'$ and corresponding sensitivity $s_k$ are time-varying; and for stochastic batches, the analog of \autoref{lem:cni-shift} incorporates the celebrated privacy amplification by subsampling phenomenon. These different settings lead to different values for the sensitivity $s$ and contractivity $c$, which in turn lead to different choices of the parameters $\lambda_k$ and $\tau$; however, we emphasize that the analysis approach is the same, and the main difference between the settings is just the elementary optimization problem over these analysis parameters, which we find the (optimal) solutions. Details in~\autoref{sec:bounds}.

\section{Improved privacy for noisy optimization algorithms}\label{sec:bounds}

Here we apply the shifted interpolation technique developed in \autoref{sec:tech} to establish improved privacy bounds for noisy gradient descent and its variants. 
We showcase the versatility of our techniques by investigating gradient descent with full-batch gradients in~\autoref{sec:noisygd}, cyclic batches in~\autoref{sec:noisycgd}, and stochastic batches in~\autoref{sec:noisysgd}. In all cases, we show convergent $f$-DP bounds for unconstrained strongly convex and constrained convex settings; the constrained strongly convex setting is similar and omitted for brevity (and the unconstrained convex setting has divergent privacy). The proofs are similar for all these different settings, based on the approach in~\autoref{sec:tech}; for brevity the proofs are deferred to \autoref{app:bounds}. See also~\autoref{app:num} for numerical illustrations of the improvements of our bounds.
\par Below, recall from~\autoref{sec:prelim} that we denote the learning rate by $\eta$, the noise rate by $\sigma$, the number of data points by $n$, the batch size by $b$, the constraint set by $\cK$, and its diameter by $D$. Throughout we denote by $c = \max\{|1 - \eta m|, |1 - \eta M|\}$ the Lipschitz constant for a step of gradient descent on $m$-strongly convex and $M$-smooth losses (c.f., \autoref{lem:contraction}).

\subsection{Noisy gradient descent}\label{sec:noisygd}

Here we consider full-batch gradient descent. For comparison, we first recall the standard $f$-DP bound implied by the composition theorem~\cite{drs22}.

\begin{theorem}\label{thm:gd} Consider loss functions with gradient sensitivity $L$. Then $\ngd$ is $\mu$-GDP where 
\[\mu = \frac{L}{n \sigma}\sqrt{t}\,.
\]

\end{theorem}
This (divergent) bound is tight without further assumptions on the losses. Below we show convergent $f$-DP bounds for $\ngd$ in the setting of strongly convex losses, and the setting of constrained convex losses. 
\begin{theorem}\label{thm:gd-sc} Consider $m$-strongly convex, $M$-smooth loss functions with gradient sensitivity $L$. Then for any $\eta \in (0, 2/M)$, $\ngd$ is $\mu$-GDP where
\[\mu = \sqrt{\frac{1-c^t}{1+c^t}\frac{1+c}{1-c}}\frac{L}{n \sigma}\,.
\]
For $\eta \in (0, 2/(M+m)]$, this bound is optimal.
\end{theorem}

\begin{theorem}\label{thm:gd-proj} Consider convex, $M$-smooth loss functions with gradient sensitivity $L$ and constraint set $\cK$ of diameter $D$. Then for any $\eta \in [0, 2/M]$ and $t \geq \frac{Dn}{\eta L}$, $\ngd$ is $\mu$-GDP where
\[\mu = \frac{1}{\sigma}\sqrt{\frac{3LD}{\eta n} + \frac{L^2}{n^2}\left\lceil \frac{Dn}{\eta L} \right\rceil}\,.
\]
\end{theorem}

\autoref{thm:gd-sc} is exactly tight in all parameters, and improves over the composition-based analysis (\autoref{thm:gd}) for all $t > 1$. See~\autoref{fig:fdp-opt}.~\autoref{thm:gd-proj} is tight up to a constant factor (see \autoref{thm:proj-lb}), and for $t > \frac{4Dn}{\eta L}$ it dominates \autoref{thm:gd} since its convergent nature outweighs the slightly suboptimal constant.

\subsection{Noisy cyclic gradient descent}\label{sec:noisycgd}

We now turn to cyclic batches. 
For simplicity, suppose that the number of batches per epoch $l = n/b$ and the number of epochs $E = t/l$ are integers. We state our results with respect to $E$ rather than $t$. For comparison, we first state the standard (divergent) $f$-DP bound implied by the composition theorem~\cite{drs22}. 
\begin{theorem}\label{thm:mgd} Consider loss functions with gradient sensitivity $L$. Then $\nmgd$ is $\mu$-GDP where 
\[\mu = \frac{L}{b \sigma}\sqrt{E}\,.
\]

\end{theorem}

Below we show convergent $f$-DP bounds for $\ncgd$ in the setting of strongly convex losses, and the setting of constrained convex losses.

\begin{theorem}\label{thm:mgd-sc} Consider $m$-strongly convex, $M$-smooth loss functions with gradient sensitivity $L$. Then for any $\eta \in (0, 2/M)$, $\nmgd$ is $\mu$-GDP where
\begin{align*}
    \mu &= \frac{L}{b\sigma}\, \sqrt{1 + c^{2l - 2}\frac{1-c^2}{(1-c^l)^2} \frac{1-c^{l(E-1)}}{1 + c^{l(E-1)}}}\,.
\end{align*}
\end{theorem}

\begin{theorem}\label{thm:mgd-proj} Consider convex, $M$-smooth loss functions with gradient sensitivity $L$ and constraint set $\cK$ of diameter $D$. Then for any $\eta \in [0, 2/M]$ and $E \geq \frac{Db}{\eta L}$, $\nmgd$ is $\mu$-GDP where
\begin{align*}
    \mu &= \frac{1}{\sigma}\sqrt{\left(\frac{L}{b}\right)^2 + \frac{3LD}{\eta b l} + \frac{L^2}{b^2l}\left\lceil \frac{Db}{\eta L}\right\rceil}\,.
\end{align*}
\end{theorem}
The convergent nature of these bounds ensures that they dominate \autoref{thm:mgd} when $\ncgd$ is run long enough. This threshold is roughly $E \approx c^{2l-2}\tfrac{1-c^2}{(1-c^l)^2}$ for \autoref{thm:mgd-sc} and $E \approx \tfrac{4Db}{\eta \ell L}$ for \autoref{thm:mgd-proj}.

\subsection{Noisy stochastic gradient descent}\label{sec:noisysgd}
Compared to $\ngd$, the privacy leakage in $\nsgd$ only occurs when the index $i^*$ is in the sampled batch. This phenomenon is known as privacy amplification by subsampling \cite{klnrs11}, which is formulated in $f$-DP as follows \cite[Definition 6]{drs22}.

\begin{definition}[Subsampling]\label{def:cp}
    For tradeoff function $f$ and $p \in [0, 1]$, define $f_p = pf + (1-p)\text{Id}$. The \emph{subsampling operator} $C_p$ (with respect to $f$) is defined as $C_p(f) = \min\{f_p, (f_p)^{-1}\}^{**}$ where $^{-1}$ denotes the (left-continuous) inverse and $^{*}$ denotes the convex conjugate. Equivalently, $C_p(f)$ is the pointwise largest tradeoff function $g$ such that $g \leq f_p$ and $g \leq (f_p)^{-1}$.
\end{definition}
For comparison, we first recall the standard $f$-DP bound based on composition~\cite[Theorem 9]{drs22}.

\begin{theorem}\label{thm:sgd}
Consider loss functions with gradient sensitivity $L$. Then $\nsgd$ is $f$-DP where
\[f = C_{b/n}(G(\frac{L}{b\sigma}))^{\otimes t}\,.
\]
\end{theorem}

This (divergent) bound is tight for $t=1$ without further assumptions on the losses. Below we show convergent $f$-DP bounds for $\nsgd$ in the setting of strongly convex losses, and the setting of constrained convex losses.

\begin{theorem}\label{thm:sgd-sc} Consider $m$-strongly convex, $M$-smooth loss functions with gradient sensitivity $L$. Then for any $\eta \in (0, 2/M)$, $\nsgd$ is $f$-DP for
\[
    f =  G(\frac{2\sqrt{2}L}{b \sigma}\frac{c^{t - \tau+1} - c^t}{1-c})
    \otimes C_{b/n}(G(\frac{2\sqrt{2} L}{b \sigma})) \otimes C_{b/n}(G(\frac{2L}{b \sigma}))^{\otimes (t - \tau)}
\]
for any $\tau = 0, 1, \dots, t-1$.
\end{theorem}

\begin{theorem}\label{thm:sgd-proj} Consider convex, $M$-smooth loss functions with gradient sensitivity $L$ and constraint set $\cK$ of diameter $D$. Then for any $\eta \in [0, 2/M]$, $\nsgd$ is $f$-DP where
\[f = G(\frac{\sqrt{2}D}{\eta \sigma \sqrt{t - \tau}}) \otimes C_{b/n}(G(\frac{2\sqrt{2}L}{b \sigma}))^{\otimes (t- \tau)}
\]
for any $\tau = 0, 1, \dots, t-1$.
\end{theorem}

Both theorems give convergent privacy by taking $t - \tau$ constant as $t \to \infty$.
In contrast,~\autoref{thm:sgd} is convergent in the regime $t = O(n^2/b^2)$ (by CLT), but yields a vacuous privacy as $t \to \infty$ for fixed $b/n$~\cite{drs22}. 
We remark that for finite but large $t$, one can set $t-\tau$ to be sufficiently large and apply CLT (\autoref{lem:clt}) to approximate the composition of $C_p(G(\cdot))$; see \autoref{lem:complim} and \autoref{subsec:cltopt}. We also remark that by choosing $t - \tau = \Theta(\tfrac{Dn}{\eta L})$, \autoref{thm:sgd-proj} recovers the asymptotically tight 
R\'enyi DP bound of \cite{at22}.

\subsection{Numerical example}\label{sec:example}

As a proof of concept, here we consider regularized logistic regression on MNIST \cite{lecun2010mnist}. We compare our results with the state-of-the-art R\'enyi DP bounds, and existing $f$-DP bounds (based on the composition theorem) which we denote as GDP Composition. For a fair comparison, we use the same algorithm $\ncgd$, with all parameters unchanged, and only focus on the privacy accounting; we focus on $\ncgd$ because it is close to standard private optimization implementations, e.g., Pytorch~\cite{pytorch} and TensorFlow~\cite{tensorflow}. Indeed, standard implementations often cycle through batches in a permuted order every epoch, somewhat similar to how $\ncgd$ cycles through batches in a fixed order.

\autoref{tab:lr-ep} demonstrates that for this problem, our privacy guarantees are tighter, enabling longer training for the same privacy budget---which helps both training and testing accuracy
(c.f., \autoref{tab:lr-acc}). 
For full details of the experiment, see \autoref{subsec:example}.
\begin{table}[H]
\centering
\caption{Privacy $\ep$ of $\ncgd$ on regularized logistic regression for $\delta = 10^{-5}$ in $(\ep, \delta)$-DP. Our results provide better privacy than both GDP Composition and RDP bounds in all cases.}
\label{tab:lr-ep}
\vskip 0.15in
\begin{tabular}{c|c|c|c}
\toprule
Epochs & \multicolumn{1}{c|}{GDP Composition} & \multicolumn{1}{c|}{RDP} & \multicolumn{1}{c}{Our Bounds} \\

 \midrule
50 & 30.51 & 5.82 & 4.34 \\
100 & 49.88 & 7.61 & 5.60 \\
200 & 83.83 & 9.88 & 7.58 \\
\bottomrule
\end{tabular}
\vskip -0.1in
\end{table}

\begin{table}[H]
\centering
\caption{Training and test accuracy (\%) of $\nmgd$ for regularized logistic regression, averaged over $10$ runs. Both the training and test accuracy improve as the number of epochs increases.}
\label{tab:lr-acc}
\vskip 0.15in
\begin{tabular}{c|c|c}
\toprule
Epochs & Training & Test \\
\midrule
50 & 89.36 $\pm$ 0.03 & 90.12 $\pm$ 0.04 \\
100 & 90.24 $\pm$ 0.03 & 90.94 $\pm$ 0.07 \\
200 & 90.85 $\pm$ 0.02 & 91.37 $\pm$ 0.08 \\
\bottomrule
\end{tabular}
\vskip -0.1in
\end{table}

\section{$f$-DP of the exponential mechanism}\label{sec:expmech}

Since we show convergent $f$-DP bounds for randomized algorithms in~\autoref{sec:bounds}, we can take the limit $t \to \infty$ to obtain $f$-DP bounds for their stationary distributions. We focus here on $\ngd$ because, up to a simple rescaling, it is equivalent to Langevin Monte Carlo ($\lmc$), one of the most well-studied sampling algorithms in the statistics literature; see, e.g.,~\cite{robert1999monte,liu2001monte,andrieu2003introduction}. Our results for (strongly) convex losses not only imply new results for (strongly) log-concave sampling for $\lmc$, but also imply $f$-DP bounds for the exponential mechanism~\cite{mt07}---a foundational concept in DP---since it is obtained from $\lmc$'s stationary distribution in the limit as the stepsize $\eta \to 0$.

 \subsection{Strongly log-concave targets}

Our optimal $f$-DP bounds for $\ngd$ immediately imply optimal\footnote{Although here we bound the optimal constants for simplicity.} $f$-DP bounds for $\lmc$.

\begin{proposition}\label{prop:lmc}
    Suppose that $F, F'$ are $m$-strongly convex and $M$-smooth, and that $F - F'$ is $L$-Lipschitz. Consider the $\lmc$ updates
    \begin{align*}
        X_{k+1} &= X_k - \eta \nabla F(X_k) + Z_{k+1} 
        \\
        X'_{k+1} &= X'_k - \eta \nabla F'(X'_k) + Z'_{k+1}
    \end{align*}
    where $X_0 = X'_0$ and $Z_{k+1}, Z'_{k+1} \sim \cN(0, 2\eta I_d)$. Then for any $\eta \in (0, 2/(M+m)]$,
    \[T(X_t, X'_t) \geq G\left(\sqrt{\frac{2-\eta m}{2}}\frac{L}{\sqrt{m}}\right)\,.
    \]
\end{proposition}
\begin{proof}
    $\lmc$ is a special case of $\ngd$ with $n = 1$, $f_1 = F, f'_1 = F'$, and $\sigma = \sqrt{2 / \eta}$. Apply \autoref{thm:gd-sc}.
\end{proof}

Taking $t \to \infty$ gives $f$-DP guarantees for the stationary distributions $\pi(\eta)$ and $\pi'(\eta)$ of these $\lmc$ chains. We also obtain $f$-DP guarantees between the exponential mechanisms $\pi \propto e^{-F}$ and $\pi' \propto e^{-F'}$ for $F$ and $F'$.

\begin{corollary}\label{cor:expmechsc}
    In the setting of \autoref{prop:lmc}, 
    \[T(\pi(\eta), \pi'(\eta)) \geq G\left(\sqrt{\frac{2-\eta m}{2}}\frac{L}{\sqrt{m}}\right)
    \]
    and $$T(\pi, \pi') \geq G\left(\frac{L}{\sqrt{m}}\right)\,.$$
\end{corollary}
\begin{proof}
    It is well-known that under these assumptions, $\lmc$ converges to its stationary distribution in total variation (TV) distance as $t \to \infty$, and the stationary distribution converges to the exponential mechanism as $\eta \to 0$, see e.g.,~\cite{chewibook}. By~\autoref{lem:tvlimit}, tradeoff functions converge under convergence in TV. 
\end{proof}

Thus, we recover the recent result~\cite[Theorem 4]{gll22} which characterizes the $f$-DP of the exponential mechanism. The proof in~\cite{gll22} is entirely different, based on the Gaussian isoperimetry inequality \cite{ledoux} rather than connecting $\lmc$ to the exponential mechanism. Our results can be viewed as algorithmic generalizations of theirs in the sense that we also obtain tight $f$-DP bounds on the iterates of $\lmc$ and its stationary distribution.

\begin{remark}[Tightness]
    As noted in \cite{gll22}, the exponential mechanism bound in \autoref{cor:expmechsc} is tight by considering $F(x) = \frac{m}{2}\norm{x}^2, F'(x) = \frac{m}{2}\norm{x - \frac{L}{m}v}^2$ (where $v$ is a unit vector) which yields $\pi = \cN(0, \frac{1}{m}I_d), \pi' = \cN(\frac{L}{m}v, \frac{1}{m}I_d)$. With the same loss functions, it is straightforward to check that this construction also shows optimality for our results on the $f$-DP of $\lmc$ and its stationary distribution; in particular, for the latter we have $\pi(\eta) = \cN(0, \frac{2}{(2-\eta m)m}I_d)$ and $\pi'(\eta) = \cN(\frac{L}{m}v, \frac{2}{(2-\eta m)m}I_d)$ which yields $T(\pi(\eta), \pi'(\eta)) = G(\sqrt{\frac{2-\eta m}{2}}\frac{L}{\sqrt{m}})$.
\end{remark}

\subsection{Log-concave targets}

A similar story holds in the setting of convex losses, although this requires a constrained setting since otherwise stationary distributions may not exist. Hence we consider projected $\ngd$ (\autoref{thm:gd-proj}), which corresponds to projected $\lmc$. As above, this leads to $f$-DP bounds for the exponential mechanism due to known TV convergence results, for projected $\lmc$ to its stationary distribution as $t \to \infty$~\cite{at23}, and from that distribution to the exponential mechanism as $\eta \to 0$~\cite{bel18}.

\begin{corollary}\label{cor:expmechconvex}
    Let $F, F'$ be convex, $M$-smooth and $L$-Lipschitz functions and $\cK$ be a convex body with diameter $D$ containing a unit ball. Then for $\pi \propto e^{-F}\bm{1}_{\cK}$ and $\pi' \propto e^{-F'}\bm{1}_{\cK}$, 
    \[T(\pi, \pi') \geq G(2 \sqrt{LD})\,.
    \]
    Furthermore, for $\eta \in (0, 2/M]$, the respective stationary distributions $\pi(\eta), \pi'(\eta)$ of the projected $\lmc$ satisfy $T(\pi(\eta), \pi'(\eta)) \geq G(\sqrt{4LD + 2\eta L^2})$.
\end{corollary}

Unlike the strongly convex case \cite{masn16, gll22}, we are unaware of any results in this setting beyond the standard analysis \cite{mt07} on the exponential mechanism. That yields $(2LD, 0)$-DP, and our result provides nontrivial improvement in privacy when $LD > 0.677$; see \autoref{subsec:expmechnum}.

\section{Discussion}\label{sec:disc}

The techniques and results of this paper suggest several directions for future work. 

\par One natural direction is whether convergent $f$-DP bounds can be shown in more general settings, e.g., (structured) non-convex landscapes, heteroscedastic or correlated noises \cite{cdpgst23}, adaptive first-order algorithms, or second-order algorithms \cite{ganesh2023faster}. 

\par A technical question is whether one can relax the $W_{\infty}$ bounds between our shifted interpolated process $\{\widetilde{X}_k\}$ and the target process $\{X_k\}$, and if this can enable tighter analyses of stochastic algorithms. While $W_{\infty}$ has traditionally been used for privacy amplification by iteration~\cite{fmtt18},~\cite{AltChe23} recently showed that some of this analysis extends to the Orlicz--Wasserstein distance, which is even necessary in some applications. 

\par Another natural direction is more computationally tractable $f$-DP bounds. Although the $f$-DP framework provides an information-theoretically lossless quantification of DP, it is often computationally burdensome, e.g., for $\nsgd$ bounds expressed as the composition of many tradeoff functions. Recent work has developed useful tools for approximation~\cite{zdls20, gopi2021numerical, zdw22}, and further developments would help practitioners who need to adhere to given privacy budgets.

\paragraph*{Acknowledgements.} We thank Sinho Chewi, Kunal Talwar, and Jiayuan Ye for insightful conversations about the literature and the anonymous reviewers for helpful comments. We also thank Jiayuan Ye for sharing helpful code for numerical experiments.

\addcontentsline{toc}{section}{References}

\newcommand{\etalchar}[1]{$^{#1}$}

\bibliographystyle{alpha}

\appendix

\section{R\'enyi DP and tradeoff functions}\label{app:tradeoff}

Here we provide helper lemmas and other relevant background about R\'enyi DP (\autoref{app:rdp}), tradeoff functions (\autoref{subsec:tradeoff}), and their convergence properties (\autoref{app:convergence}).

\subsection{R\'enyi DP}\label{app:rdp}

A popular notion of DP that is often analytically tractable is R\'enyi DP (RDP) \cite{mir17}.
\begin{definition}[RDP]
    The \emph{R\'enyi divergence} of order $\alpha > 1$ between probability distributions $P, Q$ is defined as
    \[D_\alpha(P \mmid Q) = \frac{1}{\alpha - 1} \log \int \left(\frac{dP}{dQ}(\omega)\right)^\alpha dQ(\omega)\,.
    \]
    A randomized algorithm $\mathcal{A}$ is \emph{$(\alpha, \ep)$-RDP} if for any adjacent datasets $S$ and $S'$,
    \[D_\alpha(\mathcal{A}(S) \mmid \mathcal{A}(S')) \leq \ep\,.
    \]
\end{definition}

\paragraph*{Numerical conversion.} Conversion from RDP to $(\ep, \delta)$-DP is inherently lossy and there are many proposed formulae for this. Since the RDP bounds mentioned in this text are of the form $(\alpha, \rho \alpha)$-RDP for all $\alpha > 1$, given $\rho$ and a fixed level of $\delta$ the corresponding converted value of $\ep = \ep(\alpha, \rho, \delta)$ can be found by optimizing over $\alpha$. Also, in addition to RDP, results on zero concentrated DP \cite{bs16} can be applied. Throughout, we calculate the minimum $\ep$ (aka the best bound) among the following formulae: \cite[Lemma 3.5]{bs16}, \cite[Proposition 3]{mir17}, \cite[Theorem 20]{bbg20}, and \cite[Lemma 1]{conversion}.

\subsection{Lemmas on tradeoff functions}\label{subsec:tradeoff}

Here we recall various useful facts about tradeoff functions. The first lemma records basic properties of tradeoff functions that we use repeatedly \cite[Proposition D.1]{drs22}.
\begin{lemma}[Basic properties]\label{lem:various} For tradeoff functions  $f, g_1, g_2$ and $\mu = (\mu_1, \dots, \mu_d) \in \R^d$,
\begin{itemize}
    \item[(a)] $g_1 \geq g_2 \Rightarrow f \otimes g_1 \geq f \otimes g_2$.
    \item[(b)] $f \otimes \text{Id} = \text{Id} \otimes f = f$.
    \item[(c)] $T(\cN(0, \sigma^2 I_d), \cN(\mu, \sigma^2 I_d)) = G(|\mu_1|/\sigma) \otimes \dots \otimes G(|\mu_d|/\sigma) = G(\norm{\mu}/\sigma)$.
\end{itemize}
\end{lemma}

Next, we recall tight conversion formulae from GDP to other standard notions of DP, namely $(\ep, \delta)$-DP \cite[Theorem 8]{bw18} and RDP \cite[Corollary B.6]{drs22}. 

\begin{lemma}[GDP to $(\ep, \delta)$-DP]\label{lem:fdptoapproxdp} A $\mu$-GDP mechanism is $(\ep, \delta(\ep))$-DP for all $\ep > 0$ where
\[\delta(\ep) = \Phi\left(-\frac{\ep}{\mu} + \frac{\mu}{2}\right) - e^\ep \Phi\left(-\frac{\ep}{\mu} - \frac{\mu}{2}\right)\,.
\] 
\end{lemma}

\begin{lemma}[GDP to RDP]\label{lem:fdptordp} A $\mu$-GDP mechanism is $(\alpha, \frac{1}{2}\mu^2 \alpha)$-RDP for any $\alpha > 1$. 
\end{lemma}

An appealing property of tradeoff functions is that they admit a central limit theorem (CLT) that approximates multiple compositions to GDP. In particular, the subsampled GDP can be approximated as follows \cite[Corollary 4]{drs22}.

\begin{lemma}[CLT]\label{lem:clt} Let $\mu \geq 0$ and assume that $p\sqrt{t} \to p_0$ as $t \to \infty$. Then
\[C_p(G(\mu))^{\otimes t} \to G\left(\sqrt{2}p_0 \sqrt{e^{\mu^2}\Phi(1.5\mu) + 3\Phi(-0.5\mu) - 2}\right)\,.
\]
\end{lemma}

\subsection{Convergence of tradeoff functions}\label{app:convergence}

Here we present results about the convergence of distributions as measured by tradeoff functions. The main results are \autoref{lem:tvtofdp} and \autoref{lem:tvlimit}, which state that this is equivalent to convergence in TV distance; we also present intermediate results which may be of independent interest. For notation, we use $P_n, P, Q_n, Q$ to denote probability distributions, and $\alpha, \alpha'$ to respectively denote elements in $[0, 1]$ and $(0, 1]$. Also, we use $a \vee b$ and $a \wedge b$ to respectively denote $\max\{a, b\}$ and $\min\{a, b\}$.

\begin{lemma}\label{lem:tvtofdp} The following are equivalent.
\begin{itemize}
    \item[(a)] $T(P_n, P) \to \text{Id}$.
    \item[(b)] $T(P, P_n) \to \text{Id}$.
    \item[(c)] $TV(P_n, P) \to 0$.
\end{itemize}
\end{lemma}
\begin{proof}
On one hand, if $\text{TV}(P, P_n) \to 0$, then $T(P, P_n) \to \text{Id}$ since
\[1 - \text{TV}(P, P_n) \leq \alpha + T(P, P_n)(\alpha) \leq 1\,.
\]
On the other hand, if $\text{TV}(P, P_n) \nrightarrow 0$ then by taking a subsequence $\{n'\}$ such that $\text{TV}(P, P_{n'}) \geq \ep > 0$ we know that the first equality holds for some $\alpha = \alpha_{n'}$ and thus
\[T(P, P_{n'})(\alpha_{n'}) \leq 1 - \ep - \alpha_{n'}\,.
\]
By taking a further subsequence $\{n''\}$ of $\{n'\}$ such that $\alpha_{n''} \to \alpha$ for some $\alpha$ (note that $\alpha_{n'} \leq 1 - \ep$ for all $n'$ and thus $\alpha \leq 1 - \ep$), there exists $N \in \mathbb{N}$ such that $n'' > N \Rightarrow \alpha_{n''} < \alpha + \ep/2$, from which we have
\[T(P, P_{n''})(\alpha + \frac{\ep}{2}) \leq 1 - \ep - \alpha_{n''}
\]
for all $n'' > N$. Thus
\[\liminf_{n''}T(P, P_{n''})(\alpha + \frac{\ep}{2}) \leq 1 - \frac{\ep}{2} - (\alpha + \frac{\ep}{2})\,,
\]
implying that $T(P, P_n)$ does not converge to $\text{Id}$.
\end{proof}

\begin{lemma}\label{lem:leftlim}
If $T(P_n, P) \to \text{Id}$ then for any probability distribution $Q$,
\[\lim_n T(P_n, Q)(\alpha') = T(P, Q)(\alpha')
\]
for every $\alpha' \in (0, 1]$. In particular, if $T(P, Q)(0) = 1$ then $\lim_n T(P_n, Q) = T(P, Q)$.
\end{lemma}
\begin{proof}
From \cite[Lemma A.5]{drs22} we have
\begin{align*}
    T(P, Q)(\alpha') &\geq T(P_n, Q)(1 - T(P, P_n)(\alpha')) \\
    T(P_n, Q)(\alpha) &\geq T(P, Q)(1 - T(P_n, P)(\alpha))\,.
\end{align*}
By taking $\liminf_n$ in the second line, we have $\liminf_n T(P_n, Q)(\alpha) \geq T(P, Q)(\alpha)$.

On the other hand, for any $\alpha' \in (0, 1]$ and sufficiently small $\ep > 0$ we have $1-T(P, P_n)(\alpha') \leq (\alpha' + \ep) \wedge 1 $ for all sufficiently large $n$, from which in the first line we have
\[T(P, Q)(\alpha') \geq T(P_n, Q)((\alpha' + \ep) \wedge 1)\,.
\]
Taking $\limsup_n$ (it is straightforward to check that a limit supremum of tradeoff function is continuous on $(0, 1)$) and letting $\ep \to 0$, we have $T(P, Q)(\alpha') \geq \limsup_n T(P_n, Q)(\alpha')$.
\end{proof}

\begin{remark}[Necessity of the restriction on $\alpha'$]\label{rem:alpha-technical} The restriction $\alpha' \in (0, 1]$ is necessary. For example, if $P = \delta_0$, $Q = \tfrac{1}{2}\delta_0 + \tfrac{1}{2}\delta_1$, and $P_n = (1 - \tfrac{1}{n})\delta_0 + \tfrac{1}{n}\delta_1$ (here, $\delta_x$ denotes the Dirac measure at $x$ and $pP + (1-p)Q$ denotes the mixture of $(P, Q)$ with mixing rate $(p, 1-p)$), then $\text{TV}(P_n, P) \to 0$ implies $T(P_n, P) \to \text{Id}$ and $T(P, Q)(\alpha) = \tfrac{1}{2}(1-\alpha)$, yet
\[
    T(P_n, Q)(\alpha) = \begin{cases} 1 - \frac{1}{2}n\alpha & \alpha \leq \frac{1}{n} \\ \frac{1}{2}(\frac{1-\alpha}{1 - \frac{1}{n}}) & \alpha > \frac{1}{n}\end{cases} \Rightarrow \lim_n T(P_n, Q)(\alpha) = \begin{cases}
        1 & \alpha = 0 \\ \frac{1}{2}(1-\alpha) & \alpha > 0\,.
    \end{cases}
\]
However, if the limit is switched to the second argument, then this restriction on $\alpha'$ simplifies and is unnecessary, as proven in the following lemma.
\end{remark}

\begin{lemma}\label{lem:rightlim}
If $T(Q_n, Q) \to \text{Id}$ then for any probability distribution $P$,
\[\lim_n T(P, Q_n) = T(P, Q)\,.
\]
\end{lemma}
\begin{proof}
Again, from \cite[Lemma A.5]{drs22} we have
\begin{align*}
    T(P, Q_n)(\alpha) &\geq T(Q, Q_n)(1 - T(P, Q)(\alpha)) \\
    T(P, Q)(\alpha) &\geq T(Q_n, Q)(1 - T(P, Q_n)(\alpha))\,.
\end{align*}
Taking $\liminf_n$ in the first line, we have $\liminf_n T(P, Q_n)(\alpha) \geq T(P, Q)(\alpha)$. On the other hand, we know that the limit $T(Q_n, Q) \to \text{Id}$ is uniform over $[0, 1]$---see, for example, \cite[Lemma A.7]{drs22}---and thus for any $\ep > 0$ we have $T(Q_n, Q)(\alpha) \geq 1 - \alpha - \ep$ for all $\alpha \in [0, 1]$ when $n$ is sufficiently large, from which we have
\[T(P, Q)(\alpha) \geq T(P, Q_n)(\alpha) - \ep\,.
\]
Taking $\limsup_n$ and letting $\ep \to 0$, we have $T(P, Q)(\alpha) \geq \limsup_n T(P, Q_n)(\alpha)$.
\end{proof}

\begin{lemma}\label{lem:tvlimit} If $\text{TV}(P_n, P) \to 0$ and $\text{TV}(Q_n, Q) \to 0$ then
\[\lim_n T(P_n, Q_n)(\alpha') = T(P, Q)(\alpha')
\]
for every $\alpha' \in (0, 1]$.\footnote{In \cite{ad22}, this result is stated without the restriction on $\alpha' \in (0, 1]$. However, this restriction is needed, as evidenced by the counterexample in~\autoref{rem:alpha-technical}.} In particular, if $T(P, Q)(0) = 1$ then $\lim_n T(P_n, Q_n) = T(P, Q)$.
\end{lemma}
\begin{proof}
From $T(P_n, Q_n)(\alpha) \geq T(P, Q_n)(1 - T(P_n, P)(\alpha))$ and $T(P, Q_n) \to T(P, Q)$ uniformly over $[0, 1]$ (by \autoref{lem:tvtofdp} and \autoref{lem:rightlim}), taking $\liminf_n$ we have $\liminf_n T(P_n, Q_n)(\alpha) \geq T(P, Q)(\alpha)$.

From $T(P, Q_n)(\alpha) \geq T(P_n, Q_n)(1 - T(P, P_n)(\alpha))$, for any $\alpha' \in (0, 1]$ and sufficiently small $\ep > 0$, for all sufficiently large $n$ we have $1 - T(P, P_n)(\alpha') \leq (\alpha' + \ep) \wedge 1$ and thus
\[T(P, Q_n)(\alpha') \geq T(P_n, Q_n)((\alpha' + \ep) \wedge 1)\,.
\]
Taking $\limsup_n$ and letting $\ep \to 0$, we have $T(P, Q)(\alpha') \geq \limsup_n T(P_n, Q_n)(\alpha')$.
\end{proof}

The final lemma shows how composition and limit of tradeoff functions can be combined. This is useful when, for example, we have a lower bound of the form $G(\mu) \otimes g_t$, and $g_t$ converges to $G(\nu)$ as $t \to \infty$ (e.g., by CLT), which can be approximated by the lemma as $G(\mu) \otimes g_t \approx G(\sqrt{\mu^2 +\nu^2})$.
\begin{lemma}\label{lem:complim} Let $f, g, g_n$ be tradeoff functions such that $g(\alpha) > 0$ for all $\alpha < 1$\footnote{This condition is technical and is not necessary; the same proof applies by defining $r(\delta)$ as the minimum over $\alpha \in [0, z(1-\delta)]$ where $z = \inf\{\alpha: g(\alpha) = 0\}$.} and $g_n \to g$. Then
\[\liminf_n (f \otimes g_n) \geq f \otimes g\,.
\]
\end{lemma}
\begin{proof} Fix $0 < \delta < 1$, and let $h_\delta$ be the tradeoff function defined as
\[h_\delta(\alpha) = \begin{cases} 1 - \delta - \alpha & \alpha \leq 1 - \delta \\ 0 & \alpha > 1 - \delta\,. \end{cases}
\]
Then it is known---see \cite[Equation 12]{drs22}---that for any tradeoff function $f$,
\[f \otimes h_\delta = \begin{cases} (1-\delta)f(\frac{\alpha}{1-\delta}) & \alpha \leq 1 - \delta \\ 0 & \alpha > 1 - \delta\,. \end{cases}
\]
Now we approximate $g$ by $g \otimes h_\delta$. Defining $r(\delta) = \min_{0 \leq \alpha \leq 1 - \delta}|g(\alpha) - (g \otimes h_\delta) (\alpha)|$ (the minimum exists as the function is continuous and $[0, 1 - \delta]$ is compact), we have $r(\delta) > 0$ because for any $\alpha \in [0, 1 - \delta]$
\[g(\alpha) - (g \otimes h_\delta)(\alpha) = g(\alpha) - g(\frac{\alpha}{1 - \delta}) + \delta g(\frac{\alpha}{1 - \delta})
\geq \delta g(\frac{\alpha}{1 - \delta}) \geq 0\,,
\]
where the first inequality is from $g$ decreasing; if this value is 0 then we should have $\alpha = 1-\delta$ from the second inequality, but then $g(\alpha) - g(\frac{\alpha}{1-\delta}) = g(1 - \delta) - g(1) > 0$, a contradiction.

Since the limit $g_n \to g$ is uniform, for all sufficiently large $n$ we have $g_n \geq (g - r(\delta)) \vee 0 \geq g \otimes h_\delta$, implying
\[\liminf_n (f \otimes g_n) \geq f \otimes (g \otimes h_\delta) = h_\delta \otimes (f \otimes g)\,.
\]
Then from $\lim_{\delta \to 0} h_\delta \otimes (f \otimes g) = f \otimes g$, we obtain the result.
\end{proof}

\section{Disentangling the shift in shifted divergences}\label{sec:tech:disentangle}

As mentioned in \autoref{ssec:tech}, a key motivation behind the construction of our shifted interpolated process~\eqref{eq:shift-simple} is that it demystifies the popular privacy amplification by iteration analysis for R\'enyi DP~\cite{fmtt18}, which has been used in many contexts, and in particular was recently shown to give convergent R\'enyi DP bounds for $\ngd$ and variants~\cite{at22,abt24}. Here we explain this connection. 

\par Briefly, privacy amplification by iteration arguments for R\'enyi DP use as a Lyapunov function the \emph{shifted R\'enyi divergence} $D_{\alpha}^{(z)}(P \mmid Q) = \inf_{P' : W_{\infty}(P,P') \leq z} (P' \mmid Q)$, which combines the R\'enyi divergence $D_{\alpha}$ and $\infty$-Wasserstein distance $W_{\infty}$.~\cite{fmtt18} bounds the R\'enyi DP via an argument of the form
\begin{align}
    D_{\alpha}(X_t \mmid X_t') 
    &=
    D_{\alpha}^{(z_t)} (X_t \mmid X_t') \nonumber
    \\ &\leq D_{\alpha}^{(z_{t-1})}(X_{t-1} \mmid X_{t-1}') + 
    O(a_t^2) \nonumber
    \\ &\leq D_{\alpha}^{(z_{t-2})}(X_{t-2} \mmid X_{t-2}') + 
    O(a_t^2 + a_{t-1}^2)
    \\ & \dots \nonumber
    \\ & \leq  \underbrace{D_{\alpha}^{(z_0)} (X_0 \mmid X_0')}_{=0 \text{ since } X_0 = X_0'} + 
        O\big(\sum_{k=1}^t a_k^2\big)\,
      \label{eq:tech:rdp}
\end{align}
where $z_t = 0$ and $z_{k+1} = cz_k + s - a_{k+1}$. \cite{at22,abt24} obtained convergent R\'enyi DP bounds by essentially unrolling this argument only to an intermediate time $\tau$, and then arguing that the shifted R\'enyi divergence $D_{\alpha}^{(z_{\tau})}(X_{\tau}, X_{\tau}') = 0$ if the shift $z_{\tau}$ is made sufficiently large.

\par Several open questions remained: (1) Can this argument be performed without using shifted divergences, which is an admittedly ad-hoc combination of R\'enyi divergences and Wasserstein distances? (2) Can this argument be extended beyond divergence-based relaxations of DP, namely to $f$-DP? Our paper answers both questions. 

For (1), our argument makes \emph{explicit} the surrogates implicit in the shifted divergences 
\[
D_{\alpha}^{(z_k)}(X_k \mmid X_k') = \inf_{\widetilde{X}_k \; : \; W_{\infty}(\widetilde{X}_k, X_k) \leq z_k} D_{\alpha}(\widetilde{X}_k \mmid X_k')
\]
in each intermediate iteration of the argument. Indeed, it can be shown that our shifted interpolating process $\{\widetilde{X}_k\}$, defined in~\eqref{eq:shift-simple}, gives such a random variable that achieves the value required by this shifted divergence argument. This enables re-writing the argument~\eqref{eq:tech:rdp} without any notion of \emph{shifted} divergences, in terms of the auxiliary process $\{\widetilde{X}_k\}$, as we did for $f$-DP in~\autoref{ssec:tech:new}. This completely disentangles the R\'enyi divergence and Wasserstein distance in the shifted divergence argument. 

For (2), the disentangling we achieve in (1) appears essential. The na\"ive approach of directly extending the shifted divergence argument to ``shifted tradeoffs'' $T^{(z)}(P, Q) = \sup_{P' : W_{\infty}(P,P') \leq z} (P', Q)$ runs into several subtle technical issues. For example, the argument appears to require the existence of an optimal shift $P'$. For the shifted R\'enyi argument, it suffices to find a nearly-optimal shift $D_{\alpha}^{(z)}(P \mmid Q) = \inf_{P' : W_{\infty}(P,P')\leq z} D_{\alpha}(P' \mmid Q)$, and moreover have the shift be nearly-optimal for a given R\'enyi parameter $\alpha$ but perhaps not uniformly so over all $\alpha$. Due to the more involved calculus of tradeoff functions, these issues become subtle but important problems, and have led others to state the problem of privacy amplification by iteration in $f$-DP as open, e.g.,~\cite{wsyss23}. Although the general problem of finding an optimal shift for general tradeoff functions remains open, the answer to (1)---our shifted interpolated process---explicitly constructs an optimal shift for the tradeoff functions specifically needed to analyze two contractive noisy iterations.

\section{Deferred details for \autoref{sec:bounds}}\label{app:bounds}

In this section we provide details for the proofs in~\autoref{sec:bounds}. See~\autoref{sec:tech} for a high-level overview of the analysis approach. We formalize the technique of shifted interpolated processes in a general context in~\autoref{app:shifted}, then prove the results of~\autoref{sec:noisygd},~\autoref{sec:noisycgd},~\autoref{sec:noisysgd} in~\autoref{app:noisygd},~\autoref{app:noisycgd},~\autoref{app:noisysgd}, respectively.

\subsection{Shifted interpolation for contractive noisy iterations}\label{app:shifted}

\par We begin by providing definitions that unify the presentation of the different settings. The first definition abstracts the fundamental reason underlying why noisy gradient descent and all its variants enjoy the phenomenon of privacy amplification by iteration for convex optimization---and is why the results stated in this section are for contractive noisy iterations (CNI). This is based on the observation that 
the variants of noisy gradient descent update by alternately applying contraction maps and noise convolutions~\cite[Definition 19]{fmtt18}. 

\begin{definition}[$\cni$] 
The CNI corresponding to a sequence of contractive functions $\{\phi_k\}_{k \in [t]}$, a sequence of noise distributions $\{\xi_k\}_{k \in [t]}$, and a closed and convex set $\cK$, is the stochastic process
\begin{equation}\label{eq:cni}
X_{k+1} = \Pi_{\cK}(\phi_{k+1}(X_k) + Z_{k+1})
\end{equation}
where $Z_{k+1} \sim \xi_{k+1}$ is independent of $(X_0, \dots, X_k)$. 
\end{definition}

Although $\cni(X_0, \{\phi_k\}_{k \in [t]}, \{\xi_k\}_{k \in [t]}, \cK)$ usually refers to the distribution of the final iterate $X_t$, we occasionally abuse notation by using this to refer to the entire sequence of iterates $\{X_k\}$.

The second definition abstracts the idea of shifted interpolated processes at the level of generality of CNI. See \autoref{ssec:tech:new} for an informal overview.

\begin{definition}[Shifted interpolated process]\label{def:shifted}
    Consider processes $\{X_k\}$ and $\{X_k'\}$ corresponding respectively to $\cni(X_0, \{\phi_k\}_{k \in [t]}, \{\xi_k\}_{k \in [t]}, \cK)$ and $\cni(X_0', \{\phi_k'\}_{k \in [t]}, \{\xi_k\}_{k \in [t]}, \cK)$. The \emph{shifted interpolated process} between these two CNI is the auxiliary process $\{\widetilde{X}_k\}$ satisfying $\widetilde{X}_\tau = X'_\tau$ and 
\begin{align}\label{eq:shift}
    \widetilde{X}_{k+1} = \Pi_{\cK}\left(\lambda_{k+1}\phi_{k+1}(X_k) + (1 - \lambda_{k+1})\phi'_{k+1}(\widetilde{X}_k) + Z_{k+1}\right)
\end{align}
for all $k = \tau,\dots,t-1$. Here, the noise $Z_{k} \sim \xi_k$ is coupled between the processes $\{X_k\}$ and $\{\widetilde{X}_k\}$. The parameters $\tau \in \{0, \dots, t\}$, and $\lambda_{k} \in [0,1]$ can be chosen arbitrarily, with the one restriction that $\lambda_t = 1$ so that $\widetilde{X}_t = X_t$.
\end{definition}

The upshot of shifted interpolation is the following meta-theorem. See \autoref{ssec:tech:new} for a high-level overview of this result, its proof, and its uses. Here, we state this meta-theorem in the more general framework of CNI. 

\begin{theorem}[Meta-theorem for shifted interpolation]\label{thm:shift} Let $X_t$ and $X'_t$ respectively be the output of $\cni(X_0, \{\phi_k\}_{k \in [t]}, \{\cN(0, \sigma^2 I_d)\}_{k \in [t]}, \cK)$ and $\cni(X_0, \{\phi'_k\}_{k \in [t]}, \{\cN(0, \sigma^2 I_d)\}_{k \in [t]}, \cK)$ such that each $\phi_k, \phi'_k$ is $c$-Lipschitz and $\norm{\phi_k(x) - \phi'_k(x)} \leq s_k$ for all $x$ and $k \in [t]$. Then for any intermediate time $\tau$ and shift parameters $\lambda_{\tau+1}, \dots, \lambda_{t} \in [0,1]$ with $\lambda_{t} = 1$, 
\[T(X_t, X'_t) \geq G\left(\frac{1}{\sigma}\sqrt{\sum_{k=\tau+1}^{t} a_k^2} \right)
\]
where $a_{k+1} = \lambda_{k+1}(cz_k + s_{k+1})$, $z_{k+1} = (1-\lambda_{k+1})(cz_k + s_{k+1})$, and $\norm{X_{\tau} -  X_{\tau}'} \leq z_\tau$.
\end{theorem}

To prove~\autoref{thm:shift}, we first prove two helper lemmas. The first lemma characterizes the worst-case tradeoff function between a Gaussian and its convolution with a bounded random variable. 
The lemma is tight, with equality achieved when the random variable is a constant.
\begin{lemma}\label{lem:gdpinf} For $s \geq 0$, let $R(s, \sigma) = \inf\{T(W + Z, Z): \norm{W} \leq s, Z \sim \cN(0, \sigma^2 I_d), W, Z\text{ are independent}\}$, where the infimum is taken pointwise.\footnote{The infimum of tradeoff functions is in general not a tradeoff function; however, we prove a lower bound that is in fact a tradeoff function. That is, we show that $T(W+Z, Z) \geq G(\frac{s}{\sigma})$ for all $W, Z$ satisfying the conditions in the definition of $R(s, \sigma)$. An analogous discussion also applies to \autoref{lem:cpgdpinf}.
} Then
\[R(s, \sigma) = G(\frac{s}{\sigma})\,.\]
\end{lemma}
\begin{proof}
    For any random variable $W$ with $\lVert W \rVert \leq s$, the post-processing inequality (\autoref{lem:proc}) implies
    \[T(W+Z, Z) \geq T((W, Z), (W, -W + Z))\,.
    \]
    Letting $K_1(y) = Z$ and 
$K'_1(y) = -y + Z$, 
    we have $T(K_1(y), K'_1(y)) = G(\frac{\norm{y}}{\sigma}) \geq G(\frac{s}{\sigma})$ for any fixed $y$ with $\norm{y} \leq s$ and thus by strong composition (\autoref{lem:strcomp}),
    \begin{align*}
        T((W, Z), (W, -W + Z)) &\geq T(W, W) \otimes G(\frac{s}{\sigma}) = G(\frac{s}{\sigma})\,.
    \end{align*}
    The bound is tight since equality holds with $W = sv$ for any fixed unit vector $v$.
\end{proof}

The second lemma,~\autoref{lem:cni-shift}, is the ``one-step'' version of the desired result~\autoref{thm:shift}. It uses the first lemma in its proof.

\begin{proof}[Proof of \autoref{lem:cni-shift}] For shorthand, let $Z,Z' \sim \cN(0,\sig^2 I_d)$ be independent. Then
\begin{align*}
    T(\lambda \phi(X) + (1-\lambda) \phi'(\widetilde{X}) + Z, \phi'(X') + Z') 
    & \geq T((\widetilde{X}, \lambda(\phi(X) - \phi'(\widetilde{X})) + Z), (X', Z')) \\
    & \geq T(\widetilde{X}, X') \otimes R(\lambda(cz+s), \sigma) \\
    & = T(\widetilde{X}, X') \otimes G(\frac{\lambda(cz + s)}{\sigma})\,.
\end{align*}
Above, the first step is by the post-processing inequality (\autoref{lem:proc}) for the post-processing function $(x, y) \mapsto \phi'(x) + y$. The second step is by strong composition (\autoref{lem:strcomp}), which we can apply since $\lambda(\norm{\phi(X) - \phi'(\widetilde{X})}) \leq \lambda(\norm{\phi(X) - \phi(\widetilde{X})} + \norm{\phi(\widetilde{X}) - \phi'(\widetilde{X})}) \leq \lambda(cz + s)$. The final step is by \autoref{lem:gdpinf}.
\end{proof}

\begin{proof}[Proof of~\autoref{thm:shift}]
    Let $\{\widetilde{X}_k\}$ be as in \eqref{eq:shift}. By induction, $\norm{X_k - \widetilde{X}_k} \leq z_k$ for all $k = \tau, \dots, t$ from
    \begin{align*}
        \norm{X_{k+1} - \widetilde{X}_{k+1}} &\leq (1-\lambda_{k+1}) (\lVert \phi_{k+1}(X_k) - \phi'_{k+1}(X_k) \rVert + \lVert \phi'_{k+1}(X_k) - \phi'_{k+1}(\widetilde{X}_k) \rVert) \\
        &\leq (1-\lambda_{k+1})(s_{k+1} + cz_k)\,,
    \end{align*}
    where the first line holds from \autoref{lem:projection}. Letting $Z_{k+1}, Z'_{k+1} \sim \cN(0, \sigma^2 I_d)$ be independent noises,
    \begin{align*}
        T(\widetilde{X}_{k+1}, X'_{k+1}) &\geq T(\lambda_{k+1}\phi_{k+1}(X_k)  + (1 - \lambda_{k+1}) \phi'_{k+1}(\widetilde{X}_k) + Z_{k+1}, \phi'_{k+1}(X'_k) + Z'_{k+1}) \\
        &\geq T(\widetilde{X}_k, X'_k) \otimes G(\frac{a_{k+1}}{\sigma})\,,
    \end{align*}
    where the first inequality is by the post-processing inequality (\autoref{lem:proc}) with respect to $\Pi_{\cK}$, and the second inequality is by \autoref{lem:cni-shift}. Repeating this for $k = t-1, \dots, \tau$, and using the fact that the shifted interpolated process satisfy $\widetilde{X}_t = X_t$ (from $\lambda_t = 1$) and $\widetilde{X}_\tau = X'_\tau$, we conclude the desired bound
    \[T(X_t, X'_t) = T(\widetilde{X}_t, X'_t) \geq T(\widetilde{X}_\tau, X'_\tau) \otimes G\left(\frac{1}{\sigma}\sqrt{\sum_{k=\tau + 1}^t a_k^2}\right) = G\left(\frac{1}{\sigma}\sqrt{\sum_{k=\tau + 1}^t a_k^2}\right)\,.\]
\end{proof}

\subsection{Deferred proofs for \autoref{sec:noisygd}}\label{app:noisygd}

\subsubsection{Proof of~\autoref{thm:gd-sc}}
First, we consider the following setting where the contractive factor is strictly less than 1, which corresponds to the strongly convex setting for $\ngd$.

\begin{theorem}\label{thm:cni-sc} In the setting of \autoref{thm:shift}, additionally assume that  $0 < c < 1$ and $s_k \equiv s$. Then
\[T(X_t, X'_t) \geq G\left(\sqrt{\frac{1-c^t}{1+c^t}\frac{1+c}{1-c}}\frac{s}{\sigma}\right)
\]
with equality holding if $X_0 = X'_0 = 0, \phi_k(x) = cx, \phi'_k(x) = cx + sv$ for any unit vector $v$ and $\cK = \R^d$.
\end{theorem}

\begin{proof} In \autoref{thm:shift}, we can take $\tau = 0$ and $z_\tau = 0$. Then the values of $\{\lambda_k\}, \{z_k\}, \{a_k\}$ obtained from the elementary optimization problem (\autoref{lem:gd-sc-opt}) yield the desired result. Finally, for the equality case, by direct calculation we have $X_t \sim \cN(0, \frac{1-c^{2t}}{1-c^2}\sigma^2I_d)$ and $X'_t \sim \cN(\frac{1-c^t}{1-c}sv, \frac{1-c^{2t}}{1-c^2}\sigma^2I_d)$, giving \[
T(X_t, X'_t) = G(\sqrt{\frac{1-c^t}{1+c^t}\frac{1+c}{1-c}}\frac{s}{\sigma})\,.\]
\end{proof}

\begin{lemma}\label{lem:gd-sc-opt} Given $s > 0$ and $0 < c < 1$, the optimal value of
\begin{align*}
    \text{minimize } &\sum_{k=1}^t a_k^2 \\
    \text{subject to } &z_{k+1} = (1-\lambda_{k+1})(cz_k + s), a_{k+1} = \lambda_{k+1}(cz_k + s)\\
    &z_k, a_k \geq 0, z_0 = z_t = 0 \\
    &\lambda_k \in [0, 1]
\end{align*}
is $\frac{1-c^t}{1+c^t}\frac{1+c}{1-c}s^2$.
\end{lemma}

\begin{proof} Since $z_{k+1} = (1-\lambda_{k+1})(cz_k+s)$ and $ a_{k+1} = \lambda_{k+1}(cz_k+s)$, 
 we have $z_{k+1} + a_{k+1} = s + cz_k$ for $k = 0, \dots, t-1$, from which we obtain
\[z_t = c^t z_0 + (1 + c + \dots + c^{t - 1})s - (a_t + ca_{t-1} + \dots + c^{t - 1}a_{1})\,.
\]
From $z_0 = z_t = 0$, we have
\[a_t + ca_{t-1} + \dots + c^{t - 1}a_{1} = \frac{1-c^t}{1-c}s\,.
\]
By the Cauchy-Schwarz inequality,
\[\sum_{k=1}^t a_k^2 \geq \frac{(a_t + ca_{t-1} + \dots + c^{t - 1}a_{1})^2}{\sum_{k=1}^t c^{2(t-k)}} = \frac{1-c^t}{1+c^t}\frac{1+c}{1-c}s^2
\]
where equality holds if the corresponding equality criterion of the Cauchy-Schwarz inequality is satisfied. The explicit formulae are $z_k = \tfrac{(1-c^k)(1-c^{t-k})}{(1+c^t)(1-c)}s$, $a_k = \tfrac{c^{t-k}(1+c)}{1+c^t}s$, and $
    \lambda_k = \tfrac{c^{t-k}(1-c^2)}{1-c^{t-k+2}-c^k+c^t}$.
\end{proof}

\begin{proof}[Proof of \autoref{thm:gd-sc}.] 
This follows as a direct corollary of \autoref{thm:cni-sc} with $\phi_k(x) \equiv \phi(x) = x - \frac{\eta}{n}\sum_{i=1}^n \nabla f_i(x)$ and $\phi'_k(x') \equiv \phi'(x') = x' - \frac{\eta}{n}\sum_{i=1}^n \nabla f'_i(x')$. For this application, consider parameters $c = \max\{|1 - \eta m|, |1 - \eta M|\} < 1$ (by \autoref{lem:contraction}), $s \gets \eta L / n$, and $\sigma \gets \eta \sigma$ (rescaling to simplify notation). The equality case is a straightforward calculation in the setting that $\cK = \R^d$, $X_0 = 0$, $\nabla f_i(x) = mx$ for all $i \in [n]$, and $\nabla f'_i(x)$ defined as $mx$ for $i \neq i^*$, and otherwise $mx - Lv$ for some unit vector $v$.
\end{proof}

\subsubsection{Proof of~\autoref{thm:gd-proj}}

We consider here the setting of optimization over a bounded constraint set.

\begin{theorem}\label{thm:cni-proj} In the setting of \autoref{thm:shift}, additionally assume that 
$\cK$ has a finite diameter $D$ and $s_k \equiv s$. Then for any integer $0 \leq \tau < t$,
\[T(X_t, X'_t) \geq G\left(\frac{1}{\sigma}(s\sqrt{t - \tau} + \frac{D}{\sqrt{t - \tau}})\right)\,.
\]
In particular, if $t \geq D/s$ then
\[T(X_t, X'_t) \geq G\left(\frac{1}{\sigma}\sqrt{3sD + s^2\left\lceil \frac{D}{s} \right\rceil}\right)\,.
\]
\end{theorem}

\begin{proof}
    From $\norm{X_\tau - \widetilde{X}_\tau} \leq D$ for all $\tau$, in \autoref{thm:shift} we can take $z_\tau = D$ and $c = 1$. The values of $\{\lambda_k\}, \{z_k\}, \{a_k\}$ are obtained by analyzing the following elementary optimization problem (\autoref{lem:gd-proj-opt}).
\end{proof}

\begin{lemma}\label{lem:gd-proj-opt}
    Given $s > 0$ and $D > 0$, the optimal value of 
    \begin{align*}
    \text{minimize } &\sum_{k=\tau + 1}^t a_k^2 \\
    \text{subject to } &z_{k+1} = (1-\lambda_{k+1})(z_k + s), a_{k+1} = \lambda_{k+1}(z_k + s)\\
    &z_k, a_k \geq 0, z_\tau = D, z_t = 0 \\
    &\lambda_k \in [0, 1]
\end{align*}
is $\left(s + \frac{D}{t-\tau}\right)^2(t-\tau)$. As a function of $t - \tau \in (0, \infty)$, this value is minimized when $t - \tau = D/s$.
\end{lemma}
\begin{proof}
    By adding the equations $z_{k+1} = (1-\lambda_{k+1})(z_k + s)$ and $a_{k+1} = \lambda_{k+1}(z_k + s)$ for $k = \tau, \dots, t-1$, we obtain (with $z_\tau = D$ and $z_t = 0$)
    \[a_t + a_{t-1} + \dots + a_{\tau + 1} = D + (t - \tau)s\,.
    \]
    By the Cauchy-Schwarz inequality, the minimum value of $\sum_{k=\tau + 1}^t a_k^2$ is $(s+R)^2(t-\tau)$ and is obtained when
    $z_k = R(t-k)$, $
        a_k \equiv s + R$, and $
        \lambda_k = \tfrac{s + R}{s + R(t-k+1)}$, where we use the shorthand $R := \tfrac{D}{t-\tau}$.
    The last part is straightforward from the strict convexity of the one-dimensional function $z \mapsto z\left(s + \frac{D}{z}\right)^2 = s^2 z + \frac{D^2}{z} + 2sD, z > 0$.
\end{proof}

\begin{proof}[Proof of \autoref{thm:gd-proj}] This follows by considering $s \gets \frac{\eta L}{n}$ and $\sigma \gets \eta \sigma$ as in the proof of \autoref{thm:gd-sc}.
\end{proof}

\subsection{Deferred proofs for \autoref{sec:noisycgd}}\label{app:noisycgd}
As done in the case of $\ngd$, we first characterize $\nmgd$ as a particular instance of $\cni$ and proceed to the proofs of the theorems. The following proposition holds straight from the definition; recall that $l = n / b$ is the number of batches.
\begin{proposition}\label{prop:mgd} 
 For $t = lE$ and $k = 0, 1, \dots, t-1$, let $B_1, \dots, B_l$ be a fixed partition of $[n]$ with size $b$, and define $\phi_{k+1}(x) = x - \frac{\eta}{b}\sum_{i \in B_r} \nabla f_i(x)$ and $\phi'_{k+1}(x') = x' - \frac{\eta}{b}\sum_{i \in B_r} \nabla f'_i(x')$ where $r = k +1 - l\lfloor \frac{k}{l}\rfloor$. Then the $f$-DP of $\nmgd$ is equal to that between $X_t = \cni(X_0, \{\phi_k\}_{k \in [t]}, \{\cN(0, \eta^2 \sigma^2 I_d)\}_{k \in [t]}, \cK)$ and $X'_t = \cni(X_0, \{\phi'_k\}_{k \in [t]}, \{\cN(0, \eta^2 \sigma^2 I_d)\}_{k \in [t]}, \cK)$.
\end{proposition}

\begin{proof}[Proof of \autoref{thm:mgd-sc}.] Let $j^* \in [l]$ be the index such that $i^* \in B_{j^*}$ and consider the setting in \autoref{prop:mgd}. We will establish a lower bound on $T(X_{t^*}, X'_{t^*})$ where $t^* = t + j^* - l - 1$; the lower bound on $T(X_t, X'_t)$ is then given by
\[T(X_t, X'_t) \geq T(X_{t + j^* - l}, X'_{t + j^* - l}) \geq T(X_{t + j^* - l - 1}, X'_{t + j^* - l - 1}) \otimes G(\frac{L}{b \sigma})
\]
where the first inequality holds from the post-processing inequality with $\phi_k \equiv \phi'_k$ for all $k = t+j^* - l + 1, \dots, t$, and the second inequality holds from $\norm{\phi_{t + j^* - l}(x) - \phi'_{t + j^* - l}(x)} \leq \frac{\eta L}{b}$ for all $x$ with \autoref{lem:cni}.

In general, $\phi_{k+1} = \phi'_{k+1}$ when $r = r(k) = k+1 - l \lfloor \frac{k}{l} \rfloor$ is not equal to $j^*$; otherwise $\norm{\phi_{k+1}(x) - \phi'_{k+1}(x)} \leq \frac{\eta L}{b}$ for all $x$. Thus, in \autoref{thm:shift} we can take $\tau = 0, z_\tau = 0$ and $s_{k+1} = s\bm{1}_{\{r = j^*\}}$ where $s = \frac{\eta L}{b}$. Using $\{\lambda_k\}, \{z_k\}, \{a_k\}$ obtained from the following result (\autoref{lem:mgd-sc-opt}),
\[T(X_t, X'_t) \geq G\left(\frac{1}{\sigma}\sqrt{\left(\frac{L}{b}\right)^2 + \frac{1}{\eta^2}\sum_{k=1}^{t^*} a_k^2}\right)\,.
\]
\end{proof}

\begin{lemma}\label{lem:mgd-sc-opt} Given $s > 0$ and $0 < c < 1$, let $t^* = t + j^* - l - 1$ and consider a system
\begin{align*}
    z_{k+1} &= (1-\lambda_{k+1})(cz_k + s \bm{1}_{\{r = j^*\}}), a_{k+1} = \lambda_{k+1}(cz_k + s \bm{1}_{\{r = j^*\}}) \\
    z_k, a_k &\geq 0, z_0 = z_{t^*} = 0 \\
    \lambda_k &\in [0, 1]\,.
\end{align*}
Then $\{a_k\}_{1 \leq k \leq t^*}, \{z_k\}_{0 \leq k \leq t^*}, \{\lambda_k\}_{1 \leq k \leq t^*}$ defined as
\begin{align*}
    a_k &= \begin{cases} \frac{c^{t - k + j^* - 2}}{1-c^l}\frac{1-c^2}{1+c^{t-l}}s & k \geq j^* \\ 0 & k < j^* \end{cases} \\
    z_{k+1} &= cz_k + s\bm{1}_{\{r = j^*\}} - a_{k+1}, z_0 = 0 \\
    \lambda_k &= \begin{cases} \frac{a_k}{z_k + a_k} & z_k + a_k > 0 \\ 0 & z_k + a_k = 0 \end{cases}
\end{align*}
is a solution, where $r = r(k) = k+1 - l \lfloor \frac{k}{l} \rfloor$.
\end{lemma}

\begin{proof} From the stated formulae and $z_{k+1} + a_{k+1} = cz_k + s \bm{1}_{\{r = j^*\}}$, every condition except $z_k \geq 0$ and $z_{t^*} = 0$ are straightforward to check.

If $k < j^*$ then $z_k \equiv 0$. For $k \geq j^*$, let $q$ be the integer such that $l(q-1) + j^* \leq k < lq + j^*$ and $r' = k - (l(q-1) + j^*)$. Then
\begin{align*}
    z_k &= c^{r'}(1 + c^l + \dots + c^{l(q-1)})s - (a_k + ca_{k-1} + \dots + c^{k-j^*}a_{j^*}) \\
    &= \left(c^{r'}(1 + c^l + \dots + c^{l(q-1)}) - c^{t-k+j^* - 2} \frac{1-c^2}{(1-c^l)(1+c^{t-l})}(1+c^2 + \dots + c^{2(k-j^*)})\right)s\,.
\end{align*}
For any fixed $q$, this is a decreasing function in $r'$ and thus it suffices to consider $r' = l-1$. Then
\begin{gather*}
    c^{r'}(1 + c^l + \dots + c^{l(q-1)}) - c^{t-k+j^* - 2} \frac{1-c^2}{(1-c^l)(1+c^{t-l})}(1+c^2 + \dots + c^{2(k-j^*)}) \\
    = c^{l-1}\frac{1-c^{lq}}{1-c^l} - c^{l(E-q) - 1}\frac{1-c^{2lq}}{(1-c^l)(1+c^{t-l})} = c^{l-1}\frac{1-c^{lq}}{1-c^l}\frac{1 - c^{l(E-q-1)}}{1+c^{E(l-1)}}\,,
\end{gather*}
which is nonnegative for $q \leq E-1$. Also, for $q = E-1$ this is equal to 0, implying $z_{l(E-1)+j^* - 1} = z_{t^*} = 0$ and thus $\lambda_{t^*} = 1$.
\end{proof}

\begin{proof}[Proof of \autoref{thm:mgd-proj}.] As in the proof of \autoref{thm:mgd-sc}, we establish a lower bound on $T(X_{t^*}, X'_{t^*})$ for $t^* = t + j^* - l - 1$. For any $\tau$, letting $\tau^* = j^* + l(\tau - 1)$ we have $\norm{X_{\tau^*} - \widetilde{X}_{\tau^*}} \leq D$. Thus in \autoref{thm:shift} we can take $z_{\tau^*} = D, s_{k+1} = s\bm{1}_{\{r = j^*\}} (s = \frac{\eta L}{b})$ and $c = 1$. The sequences $\{\lambda_k\}, \{z_k\}, \{a_k\}$ can be chosen as in the following result (\autoref{lem:mgd-proj-opt}), which yields a bound of
\[T(X_t, X'_t) \geq G\left(\frac{1}{\sigma}\sqrt{\left(\frac{L}{b}\right)^2 + \frac{1}{\eta^2}\sum_{k=\tau^* + 1}^{t^*} a_k^2}\right) \geq G\left(\frac{1}{\sigma}\sqrt{\left(\frac{L}{b}\right)^2 + \frac{(D/\eta + L(E - \tau)/b)^2}{l(E - \tau)}} \right)
\]
by \autoref{prop:mgd}. Optimizing over the choice of $E - \tau$ can be done similarly as in \autoref{thm:cni-proj}; in particular, one can take $E - \tau = \lceil \frac{Db}{\eta L} \rceil$ when $E \geq \frac{Db}{\eta L}$.
\end{proof}

\begin{lemma}\label{lem:mgd-proj-opt} Given $s > 0$ and $D > 0$, let $t^* = t + j^* - l - 1, \tau^* = j^* + l(\tau - 1)$ and consider a system
\begin{align*}
    z_{k+1} &= (1-\lambda_{k+1})(z_k + s \bm{1}_{\{r = j^*\}}), a_{k+1} = \lambda_{k+1}(z_k + s \bm{1}_{\{r = j^*\}}) \\
    z_k, a_k &\geq 0, z_{\tau^*} = D, z_{t^*} = 0 \\
    \lambda_k &\in [0, 1]\,.
\end{align*}
Then $\{a_k\}_{\tau^* + 1 \leq k \leq t^*}, \{z_k\}_{\tau^* \leq k \leq t^*}, \{\lambda_k\}_{\tau^* + 1 \leq k \leq t^*}$ defined as
\begin{align*}
    a_k &\equiv \frac{D + s(E-\tau)}{l(E-\tau)} \\
    z_{k+1} &= z_k + s\bm{1}_{\{r = j^*\}} - a_{k+1}, z_{\tau^*} = D \\
    \lambda_k &= \begin{cases} \frac{a_k}{z_k + a_k} & z_k + a_k > 0 \\ 0 & z_k + a_k = 0 \end{cases}
\end{align*}
is a solution, where $r = r(k) = k+1 - l \lfloor \frac{k}{l} \rfloor$.
\end{lemma}
\begin{proof}
As in the proof of \autoref{lem:mgd-sc-opt}, it suffices to check that $z_k \geq 0$ and $z_{t^*} = 0$. Let $q \geq \tau$ be the integer such that $l(q-1) + j^* \leq k < lq + j^*$ and $r' = k - (l(q-1) + j^*)$. Then
\begin{align*}
    z_k &= D + (q - \tau + 1)s - (l(q-\tau) + r' + 1) \frac{D + s(E-\tau)}{l(E-\tau)} \\
    &\geq D + (q - \tau + 1)s - (q-\tau + 1) \frac{D + s(E-\tau)}{(E-\tau)} \\
    &= D(1 - \frac{q-\tau+1}{E - \tau}) \geq 0\,,
\end{align*}
where the inequality is from that the first line is minimized when $r' = l-1$ for any fixed $q$. Also, $z_k = 0$ and $\lambda_k = 1$ when $r' = l-1$ and $q = E-1$, i.e., $k = t + j^* - l - 1 = t^*$.
\end{proof}

\subsection{Deferred proofs for \autoref{sec:noisysgd}}\label{app:noisysgd}

\subsubsection{\nsgd \ as stochastic version of CNI}
We first revisit the composition bound (\autoref{thm:sgd}). The key point in this proof relevant to our new results is the following formulation, which can be considered as a stochastic version of CNI \eqref{eq:cni} with each map $x \mapsto \psi_s(x), \phi_s(x), \phi'_s(x)$ being contractive. 

\begin{equation}\label{eq:sgdascni}
\begin{aligned}
    X_{k+1} &= \Pi_{\cK}(\psi_{S_k}(X_k) + V_k(\phi_{S_k} - \psi_{S_k})(X_k) + Z_{k+1}) \\
    X'_{k+1} &= \Pi_{\cK}(\psi_{S'_k}(X'_k) + V'_k(\phi'_{S'_k} - \psi_{S'_k})(X'_k) + Z'_{k+1})
\end{aligned}
\end{equation}

\begin{proof}[Proof of \autoref{thm:sgd}.] Let $X_k$ and $X_{k+1}$ respectively be the $k$-th and $(k+1)$-th iterate of $\nsgd$ with losses $\{f_i\}_{i \in [n]}$, and similarly define $X'_k$ and $X'_{k+1}$ for $\nsgd$ with losses $\{f'_i\}_{i \in [n]}$. It suffices to show
\[T(X_{k+1}, X'_{k+1}) \geq T(X_k, X'_k) \otimes C_{b/n}(G(\frac{L}{b\sigma}))\,.
\]
For the corresponding random batch $B_k$, we sample a random pair of set and element $S_k = (R_k, C_k)$ as described below; see~\cite[\S3.2]{abt24} for a discussion of this construction. This $S_k$ will be here and after used as a representation for the random batch $B_k$.
    \begin{enumerate}
        \item Sample a set $A_1$ of size $b$ in $[n] \setminus \{i^*\}$ uniformly at random.
        \item Sample an element $A_2$ from $A_1$ uniformly at random. This element will serve as a candidate to be (potentially) replaced by $i^*$.
        \item Let $R_k = A_1 \setminus \{A_2\}, C_k = A_2$.
    \end{enumerate}
    Finally, let $V_k \sim \text{Ber}(p)$ be a Bernoulli random variable with success probability $p = b/n$, which serves as an indicator denoting whether $i^* \in B_k$ (i.e., $V_k = 1$) or not (i.e., $V_k = 0$). Then
    \[B_k = \begin{cases} R_k \cup \{C_k\} & V_k = 0 \\ R_k \cup \{i^*\} & V_k = 1 \end{cases}
    \]
    is a valid sampling procedure for $B_k$ (i.e., the marginal distribution of $B_k$ is uniform over size $b$ subsets of $[n]$). These can be defined similarly for $X'_{k+1}$ as $B'_k, V'_k$ and $S'_k$.
    
    The reason for formulating this alternative sampling scheme is to separate the subsampling part---which only depends on whether the index $i^*$ is included in the batch---from the rest of the information on the batch. In particular, in \eqref{eq:sgdascni}, $S_k$ and $V_k$ are independent and $V_k$ is still distributed as $\text{Ber}(p)$ after conditioning on $S_k$. 

    Now for a pair of set and element $S = (R, C)$, define
    \begin{equation}\label{eq:phipsi}
    \begin{aligned}
        \phi_S(x) &= x - \frac{\eta}{b}(\nabla f_{i^*} + \sum_{i \in R} \nabla f_i)(x) \\
        \phi'_S(x) &= x - \frac{\eta}{b}(\nabla f'_{i^*} + \sum_{i \in R} \nabla f_i)(x) \\
        \psi_S(x) &= x - \frac{\eta}{b}\sum_{i \in R \cup \{C\}} \nabla f_i(x)\,.
    \end{aligned}
    \end{equation}

    Then the updates for $X_{k+1}$ and $X'_{k+1}$ can be respectively written as \eqref{eq:sgdascni}, where $Z_{k+1}, Z'_{k+1} \sim \cN(0, \eta^2 \sigma^2 I_d)$ are independent of anything else. Now the tradeoff function between $X_{k+1}$ and $X'_{k+1}$ satisfies
    \[T(X_{k+1}, X'_{k+1}) \geq T((X_k, S_k, V_k(\phi_{S_k} - \psi_{S_k})(X_k) + Z_{k+1}), (X'_k, S'_k, V'_k(\phi'_{S'_k} - \psi_{S'_k})(X'_k) + Z'_{k+1}))
    \]
    by the post-processing inequality with respect to $(x, s, y) \mapsto \Pi_{\cK}(\psi_s(x) + y)$. For any fixed realization $(x, s) = (x, (r, c))$ of the first two arguments, we find a lower bound on
    \begin{equation}\label{eq:cpg}
        T(V_k(\phi_s - \psi_s)(x) + Z_{k+1}, V'_k(\phi'_s - \psi_s)(x) + Z'_{k+1})\,.
    \end{equation}
    In fact, this is tradeoff function of the subsampled Gaussian mechanism as presented in \cite[Theorem 9]{drs22}. To see this, we construct a new private setting as follows:
    \begin{itemize}
        \item Datasets: $S = \{y_1, y_2, \dots, y_n\}, S' = \{y'_1, y_2, \dots, y_n\}$ where $y'_1, y_1, \dots, y_n$ are distinct alphabets. Note that the ``datasets'' here are considered only for this part of the proof and are irrelevant with the original datasets in the private optimization setting.
        \item Mechanisms:
        \begin{itemize}
            \item[(a)] $\texttt{Sample}_b$: From a set of size $n$, sample a set of size $b$ uniformly at random.
            \item[(b)] $\texttt{M}$: Given a set $R$ of size $b$, output $\theta(R) + \cN(0, \eta^2 \sigma^2 I_d)$ where
            \[\theta(R) = \begin{cases} (\phi_s - \psi_s)(x) & y_1 \in R, y'_1 \notin R \\ (\phi'_s - \psi_s)(x) & y_1 \notin R, y'_1 \in R \\ 0 & \text{else}\,. \end{cases}
            \]
        \end{itemize}
    \end{itemize}
   
    Then a lower bound $f$ on \eqref{eq:cpg} is equivalent to $\texttt{M} \circ \texttt{Sample}_b$ being $f$-DP (when considered as being applied to $S$ and $S'$). Note that from
    \begin{align*}
        \phi_s - \psi_s &= \frac{\eta}{b}(\nabla f_{c} - \nabla f_{i^*}) \\
        \phi'_s - \psi_s &= \frac{\eta}{b}(\nabla f_{c} - \nabla f'_{i^*}) \\
        \phi_s - \phi'_s &= \frac{\eta}{b}(\nabla f'_{i^*} - \nabla f_{i^*})\,,
    \end{align*}
    $\theta$ has ($l_2$-)sensitivity $\frac{\eta L}{b}$ and thus $\texttt{M}$ is $G(\frac{L}{b\sigma})$-DP by \cite[Theorem 1]{drs22}. Then by \cite[Theorem 9]{drs22}, $\texttt{M} \circ \texttt{Sample}_b$ is $C_{b/n}(G(\frac{L}{b\sigma}))$-DP. Thus,
    \begin{align*}
        &T((X_k, S_k, V_k(\phi_{S_k} - \psi_{S_k})(X_k) + Z_{k+1}), (X'_k, S'_k, V'_k(\phi'_{S'_k} - \psi_{S'_k})(X'_k) + Z'_{k+1})) \\
        &\geq T((X_k, S_k), (X'_k, S'_k)) \otimes C_{b/n}(G(\frac{L}{b\sigma})) \\
        &=T(X_k, X'_k) \otimes C_{b/n}(G(\frac{L}{b\sigma}))
    \end{align*}
    where the equality is from that $S_k (S'_k)$ is independent of $X_k (X'_k)$, and that $S_k$ and $S'_k$ have the same distribution.
    \end{proof}

    \paragraph*{One step optimality.}

    Now we show the optimality of \autoref{thm:sgd} for $t = 1$, i.e.,
    \[T(X_1, X'_1) \geq C_{b/n}(G(\frac{L}{b\sigma}))\,.
    \]
    Let $X_0 = X'_0 = 0, \cK = \R^d$ and for $\lambda \in [0, 1]$, consider the gradients
    \begin{align*}
        \nabla f_i = \nabla f'_i &= 0 \\
        \nabla f_{i^*} &= (1-\lambda) Lu \\
        \nabla f'_{i^*} &= -\lambda Lu
    \end{align*}
    where $u$ is a unit vector. Then 
    \begin{align*}
        T(X_1, X'_1) &= T(-(1-\lambda)\frac{\eta L}{b}uV_0 + \cN(0, \eta^2 \sigma^2 I_d), \lambda\frac{\eta L}{b}uV'_0 + \cN(0, \eta^2 \sigma^2I_d)) \\
        &= T(-(1-\lambda)\frac{L}{b\sigma}V_0 + \cN(0, 1), \lambda \frac{L}{b\sigma}V'_0 + \cN(0, 1))
    \end{align*}
    where $V_0, V'_0 \sim \text{Ber}(p)$. Denoting the corresponding tradeoff function as $f^{(\lambda)}$, a valid lower bound for $T(X_1, X'_1)$ is (pointwise) at most $\inf_{\lambda \in [0, 1]} f^{(\lambda)}$ and thus it suffices to show that $\inf_{\lambda \in [0, 1]} f^{(\lambda)} = C_{b/n}(G(\frac{L}{b\sigma}))$. Now the rest of the proof is a combination of following facts.
    \begin{itemize}
        \item[(a)] $f^{(1)}(\alpha) \geq C_{b/n}(G(\frac{L}{b\sigma}))(\alpha)$ with equality holding for all $\alpha \in [0, \Phi(-\frac{L}{2b\sigma})]$.
        \item[(b)] $f^{(0)}(\alpha) \geq C_{b/n}(G(\frac{L}{b\sigma}))(\alpha)$ with equality holding for all $\alpha \in [p\Phi(-\frac{L}{2b\sigma}) + (1-p)\Phi(\frac{L}{2b\sigma}), 1]$.
        \item[(c)] For $\lambda \in (0, 1)$, $f^{(\lambda)}(\alpha) \geq C_{b/n}(G(\frac{L}{b\sigma}))(\alpha)$ with equality holding at $\alpha = p\Phi(-\frac{L}{2b\sigma}) + (1-p)\Phi((\frac{1}{2} - \lambda)\frac{L}{b\sigma})$ (note that as $\lambda$ varies, this covers the range of $\alpha$ at which $C_{b/n}(G(\frac{L}{b\sigma}))$ is linear with slope $-1$ and interpolates the boundaries in (a) and (b)).
    \end{itemize}

    \begin{figure}[H]
        \centering
        \includegraphics[width=0.35\textwidth]{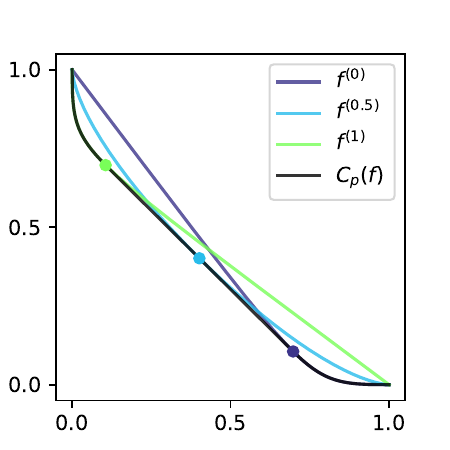}
        \caption{Illustration of $C_p(G(\frac{L}{b\sigma}))$ and $f^{(\lambda)}$, for $\lambda \in \{0, 0.5, 1\}$ with $p = 0.25, L/(b\sigma) = 2.5$.}
        \label{fig:optproof}
    \end{figure}
    
    The first two facts are straightforward from \autoref{def:cp}, with $G(\mu)_p = T(\cN(0, 1), p\cN(\mu, 1) + (1-p)\cN(0, 1))$ and $(f^{(0)})^{-1} = f^{(1)}$.\footnote{In fact, since tradeoff functions are convex, (a) and (b) are enough to conclude that $C_p(G(\mu))$ is the best tradeoff function bound; (c) provides an additional explanation on the linear part of $C_p(G(\mu))$. See \autoref{fig:optproof}.} For (c), note that as a mixture of one-dimensional Gaussians the likelihood ratio between the two distributions is monotone and thus for any $z \in \R$, with $\alpha = 1 - (1-p)\Phi(z) - p\Phi(z + (1-\lambda)\frac{L}{b\sigma})$ we have
    \[f^{(\lambda)}(\alpha) = (1-p)\Phi(z) + p\Phi(z - \lambda \frac{L}{b\sigma})\,.
    \]
    Thus from (here $\varphi$ denotes the probability density function of $\cN(0, 1)$)
    \begin{align*}
        \frac{d\alpha}{dz} &= -(1-p)\varphi(z) - p\varphi(z + (1-\lambda)\frac{L}{b\sigma}) \\
        \frac{df^{(\lambda)}(\alpha)}{dz} &= (1-p)\varphi(z) + p\varphi(z - \lambda\frac{L}{b\sigma})\,,
    \end{align*}
    at $z = (\lambda - \frac{1}{2})\frac{L}{b \sigma}$ we have $\alpha = p\Phi(-\frac{L}{2b \sigma}) + (1-p)\Phi((\frac{1}{2}-\lambda)\frac{L}{b\sigma})$ where $\alpha + f^{(\lambda)}(\alpha) = (1+p)\Phi(-\frac{L}{2 b \sigma}) + (1-p)\Phi(\frac{L}{2 b \sigma})$ and $\frac{df^{(\lambda)}}{d\alpha}(\alpha) = \frac{df_\lambda(\alpha)}{dz} / \frac{d\alpha}{dz} = -1$. This implies that $f^{(\lambda)}$ is tangent to $C_{b/n}(G(\frac{L}{b\sigma}))$ at the point, and (c) follows by \autoref{prop:fdpchar}.

\subsubsection{Proofs of new results}

As in the case of $\ngd$ (\autoref{lem:gdpinf}), we start by establishing a lower bound for tradeoff function between convolutions of Gaussian random variables with bounded random variables---now including the subsampling.
\begin{lemma}\label{lem:cpgdpinf} For $s \geq 0$ and $p = b/n$, let
    \begin{align*} R(s, \sigma, p) = \inf \{T(VW + Z, VW' + Z): &\ V \sim \text{Ber}(p), \norm{W}, \norm{W'} \leq s, Z \sim \cN(0, \sigma^2 I_d)\}\, 
    \end{align*}
    where the infimum is taken pointwise and is over independent $V, W, W', Z$. Then $R(s, \sigma, p) \geq C_p(G(\frac{2s}{\sigma}))$.\footnote{We conjecture that a strictly better lower bound holds, which corresponds to the case when $W$ and $W'$ are constant vectors aligned in the opposite direction, i.e., $R(s, \sigma, p) = T(p\cN(-\frac{s}{\sigma}, 1) + (1-p)\cN(0, 1), p\cN(\frac{s}{\sigma}, 1) + (1-p)\cN(0, 1))$.}
\end{lemma}
\begin{proof} The proof is fairly similar to the subsampling part in the proof of \autoref{thm:sgd}. Let $V, W, W', Z$ be as in the definition of $R(s, \sigma, p)$, and consider the following private setting:
\begin{itemize}
    \item Datasets: $S = \{y_1, y_2, \dots, y_n\}, S' = \{y'_1, y_2, \dots, y_n\}$ where $y'_1, y_1, \dots, y_n$ are distinct alphabets.
    \item Mechanisms:
    \begin{itemize}
            \item[(a)] $\texttt{Sample}_b$: From a set of size $n$, sample a set of size $b$ uniformly at random.
            \item[(b)] $\texttt{M}$: Given a set $R$ of size $b$, output $\theta(R) + Z$ where
            \[\theta(R) = \begin{cases} W & y_1 \in R, y'_1 \notin R \\ W' & y_1 \notin R, y'_1 \in R \\ 0 & \text{else}\,. \end{cases}
            \]
    \end{itemize}
\end{itemize}
From $\norm{W}, \norm{W'}, \norm{W - W'} \leq 2s$, $\theta$ has sensitivity $2s$ and thus $\texttt{M}$ is a $G(\frac{2s}{\sigma})$-DP mechanism by \cite[Theorem 1]{drs22}. By \cite[Theorem 9]{drs22}, $\texttt{M} \circ \texttt{Sample}_b$ is $C_p(G(\frac{2s}{\sigma}))$-DP, which is equivalent to $T(VW + Z, VW'+Z) \geq C_p(G(\frac{2s}{\sigma}))$.
\end{proof}

\begin{figure}[H]
    \centering
    \includegraphics[width = 0.7\textwidth]{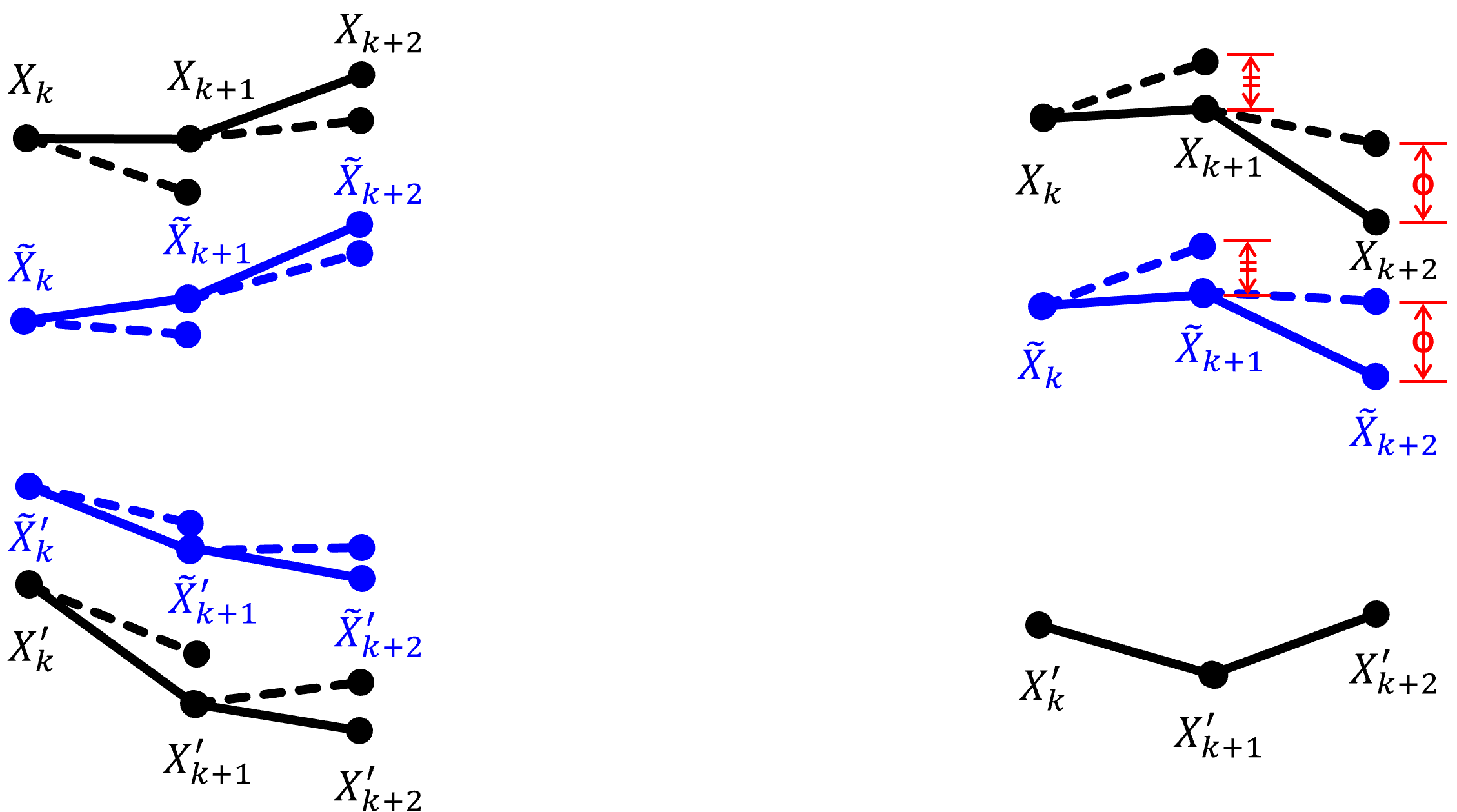}
    \caption{Illustration of shifted interpolated processes in the proofs of \autoref{thm:sgd-sc} (left) and \autoref{thm:sgd-proj} (right). The solid lines denote the updates based on the realized values of $\{V_k\}$, and the dashed lines denote the alternative updates based on their unrealized values; each interpolated process uses the same (coupled) values of $\{V_k\}$ as expressed in the figure. In \autoref{thm:sgd-sc}, we build two processes, each of which tracks its corresponding original process. In \autoref{thm:sgd-proj}, only one process is built and it inherits the identical deviation based on the realizations of $\{V_k\}$.}
    \label{fig:shift-int-sgd-sc}
\end{figure}

Now we proceed to the proofs of the new results. The key point here is that we build shifted interpolated processes by not only coupling the noise $Z_{k+1}$ but also the subsampling indicator $V_k$; see \autoref{fig:shift-int-sgd-sc}. For the strongly convex and smooth setting, we state and prove a general theorem that allows one to choose the sequences of shift and sensitivity.

\begin{theorem}\label{thm:sgd-sc-gen} Consider $m$-strongly convex, $M$-smooth loss functions with gradient sensitivity $L$. Then for any $\eta \in (0, 2/M)$, $\nsgd$ is $f$-DP where
\[f = G(\frac{2\sqrt{2}cz_{t-1}}{\eta \sigma}) \otimes \bigotimes_{k=0}^{t-1}C_{b/n}(G(\frac{2a_k}{\eta \sigma})) 
\]
for any sequence $\{z_k\}_{0 \leq k \leq t-1}, \{a_k\}_{0 \leq k \leq t-1}$ such that $z_0 = 0, a_0 = \frac{\sqrt{2}\eta L}{b}, a_k \leq \frac{\eta L}{b}$ for all $k \geq 1$ and $z_{t-1} = \frac{1-c^{t-1}}{1-c}\frac{\eta L}{b} - \sum_{k=1}^{t-1} c^{t-k-1}a_k$ and  where $c = \max\{|1 - \eta m|, |1 - \eta M|\}$.
\end{theorem}

\begin{proof}
    As in \eqref{eq:sgdascni} and \eqref{eq:phipsi}, the iterates of $\nsgd$ with respect to $\{f_i\}_{i \in [n]}$ and $\{f'_i\}_{i \in [n]}$ are
    \begin{align*}
        X_{k+1} &= \Pi_{\cK}(\psi_{S_k}(X_k) + V_k(\phi_{S_k} - \psi_{S_k})(X_k) + Z_{k+1}) \\
        X'_{k+1} &= \Pi_{\cK}(\psi_{S'_k}(X'_k) + V'_k(\phi'_{S'_k} - \psi_{S'_k})(X'_k) + Z'_{k+1})\,,
    \end{align*}
    where $Z_{k+1}, Z'_{k+1} \sim \cN(0, \eta^2 \sigma^2 I_d)$. Now consider shifted interpolated processes defined as
    \begin{align*}
        \widetilde{X}_{k+1} &= \Pi_{\cK}(\psi_{S_k}(\widetilde{X}_k) + \lambda_{k+1}V_k(\phi_{S_k}(X_k) - \psi_{S_k}(\widetilde{X}_k)) + Z_{k+1}) \\
        \widetilde{X}'_{k+1} &= \Pi_{\cK}(\psi_{S'_k}(\widetilde{X}'_k) + \lambda_{k+1}V'_k(\phi'_{S'_k}(X'_k) - \psi_{S'_k}(\widetilde{X}'_k)) + Z'_{k+1})\,,
    \end{align*}
    with $\widetilde{X}_0 = \widetilde{X}'_0 = X_0$ and 
    $\lambda_k = \tfrac{a_k}{z_k + a_k} \cdot \bm{1}_{\{z_k + a_k > 0\}}$
    for $\{z_k\}_{0 \leq k \leq t-1}$ and $\{a_k\}_{0 \leq k \leq t-1}$ such that $z_0 = 0, a_0 = \frac{\sqrt{2} \eta L}{b}$ and $z_{k+1} = cz_k + \frac{\eta L}{b} - a_{k+1}$ for all $k \geq 0$. Then inductively $\norm{\widetilde{X}_k - X_k} \leq z_k$ for all $k$ from
    \[\norm{\widetilde{X}_{k+1} - X_{k+1}} \leq \begin{cases} \norm{\psi_{S_k}(X_k) - \psi_{S_k}(\widetilde{X}_k)} \leq cz_k & V_k = 0 \\ \norm{(1-\lambda_{k+1})(\phi_{S_k}(X_k) - \psi_{S_k}(\widetilde{X}_k))} \leq cz_k + \frac{\eta L}{b} - a_{k+1} & V_k = 1 \end{cases}
    \]
    and $\norm{\lambda_{k+1}(\phi_{S_k}(X_k) - \psi_{S_k}(\widetilde{X}_k))} \leq a_{k+1}$; similar results hold for $\{X'_k\}$. Thus as in the proof of \autoref{thm:sgd} (see also \autoref{thm:shift}), with \autoref{lem:cpgdpinf}
    \[T(\widetilde{X}_{t-1}, \widetilde{X}'_{t-1}) \geq \bigotimes_{k=1}^{t-1}C_{b/n}(G(\frac{2a_k}{\eta \sigma}))\,.
    \]
    To relate this with $T(X_t, X'_t)$, note that there is no choice of $\lambda_t$ that yields $\widetilde{X}_t = X_t$. 
    Instead, we can proceed as follows: write down the corresponding update (before taking the projection) as
      \begin{align*}
        &\;\psi_{S_{t-1}}(X_{t-1}) + V_{t-1}(\phi_{S_{t-1}} - \psi_{S_{t-1}})(X_{t-1}) + Z_t
        \\ =&\; \psi_{S_{t-1}}(\widetilde{X}_{t-1}) + \psi_{S_{t-1}}(X_{t-1}) - \psi_{S_{t-1}}(\widetilde{X}_{t-1}) + Z^{(1)}_t + V_{t-1}(\phi_{S_{t-1}} - \psi_{S_{t-1}})(X_{t-1}) + Z^{(2)}_t
    \end{align*}
    where $Z^{(1)}_t, Z^{(2)}_t \sim \cN(0, \frac{\eta^2 \sigma^2}{2}I_d)$\footnote{In general, we can split the noise into $Z_t = Z^{(1)}_t + Z^{(2)}_t$ where $Z^{(1)}_t \sim \cN(0, \frac{\eta^2 \sigma^2}{\alpha^2}I_d)$ and $Z^{(2)}_t \sim \cN(0, \frac{\eta^2 \sigma^2}{\beta^2}I_d)$ are independent and $1/\alpha^2 + 1/\beta^2 = 1$. Then the part $G(\frac{2\sqrt{2}cz_{t-1}}{\eta \sigma}) \otimes C_{b/n}(G(\frac{2\sqrt{2} L}{b \sigma}))$ in the last line of the proof is replaced with $G(\frac{2\alpha cz_{t-1}}{\eta \sigma}) \otimes C_{b/n}(G(\frac{2\beta L}{b \sigma}))$.} are independent, $\psi_{S_{t-1}}(X_{t-1}) - \psi_{S_{t-1}}(\widetilde{X}_{t-1})$ is bounded by $cz_{t-1}$ and $(\phi_{S_{t-1}} - \psi_{S_{t-1}})(X_{t-1})$ is bounded by $\frac{\eta L}{b}$. Then
    \begin{align*}
        &T(X_t, X'_t) \\
        &\geq T((\widetilde{X}_{t-1}, S_{t-1}, \psi_{S_{t-1}}(X_{t-1}) - \psi_{S_{t-1}}(\widetilde{X}_{t-1}) + Z^{(1)}_t), (\widetilde{X}'_{t-1}, S'_{t-1}, \psi_{S'_{t-1}}(X'_{t-1}) - \psi_{S'_{t-1}}(\widetilde{X}'_{t-1}) + Z^{(1)'}_t)) \\
        &\quad\otimes R(\frac{\eta L}{b}, \frac{\eta \sigma}{\sqrt{2}}, b/n) \\
        &\geq T((\widetilde{X}_{t-1}, S_{t-1}), (\widetilde{X}'_{t-1}, S'_{t-1})) \otimes R(cz_{t-1}, \frac{\eta \sigma}{\sqrt{2}}, 1) \otimes R(\frac{\eta L}{b}, \frac{\eta \sigma}{\sqrt{2}}, b/n) \\
        &\geq T(\widetilde{X}_{t-1}, \widetilde{X}'_{t-1}) \otimes G(\frac{2\sqrt{2}cz_{t-1}}{\eta \sigma}) \otimes C_{b/n}(G(\frac{2\sqrt{2} L}{b \sigma}))\,.
    \end{align*}
\end{proof}

In this formulation, optimizing over the sequences $\{z_k\}$ and $\{a_k\}$ is intractable because of the analytically complicated nature of the subsampled operator and composition of tradeoff functions. Heuristically, when $b/n$ is small, each individual $C_{b/n}(G(\cdot))$ is very close to $\text{Id}$ and the most substantial factor is the GDP part. In this sense, sequences that make $z_{t-1}$ small can be considered as a reasonable choice.

\begin{proof}[Proof of \autoref{thm:sgd-sc}] Consider $a_{t-1} = \dots = a_\tau = \frac{\eta L}{b}$ and $a_k = 0$ for all $1 \leq k < \tau$ in \autoref{thm:sgd-sc-gen}.    
\end{proof}

\begin{proof}[Proof of \autoref{thm:sgd-proj}]For the iterates \eqref{eq:sgdascni} and \eqref{eq:phipsi}, consider the shifted interpolated process 
\[\widetilde{X}_{k+1} = \Pi_{\cK}(\psi_{S_k}(\widetilde{X}_k) + \lambda_{k+1}(\psi_{S_k}(X_k) - \psi_{S_k}(\widetilde{X}_k)) + V_k(\phi_{S_k} - \psi_{S_k})(X_k) + Z_{k+1})
\]
where $\lambda_{k+1} = \frac{1}{t-k}$ and $\widetilde{X}_\tau = X'_\tau$. Then for any $k \geq \tau$, $\norm{\widetilde{X}_k - X_k} \leq z_k$ and $\norm{\lambda_{k+1}(\psi_{S_k}(X_k) - \psi_{S_k}(\widetilde{X}_k))} \leq a_{k+1}$ where
\begin{align*}
    z_k &= \frac{D}{t-\tau}(t - k) \\
    a_{k+1} &\equiv \frac{D}{t-\tau}\,.
\end{align*}
The first inequality is inductively from $\norm{\widetilde{X}_\tau - X_\tau} = \norm{X'_\tau - X_\tau} \leq D$ and
\[\norm{\widetilde{X}_{k+1} - X_{k+1}} \leq (1-\lambda_{k+1})\norm{\widetilde{X}_k - X_k} \leq z_{k+1}\,.
\]
The second inequality is from $\norm{\lambda_{k+1}(\psi_{S_k}(X_k) - \psi_{S_k}(\widetilde{X}_k))} \leq \lambda_{k+1}z_k = a_{k+1}$. Also, note that $\widetilde{X}_t = X_t$. As in the proof of \autoref{thm:sgd-sc}, we can write down as
\begin{align*}
    &\;\psi_{S_k}(\widetilde{X}_k) + \lambda_{k+1}(\psi_{S_k}(X_k) - \psi_{S_k}(\widetilde{X}_k)) + V_k(\phi_{S_k}(X_k) - \psi_{S_k}(X_k)) + Z_{k+1}
    \\ =&\;\psi_{S_k}(\widetilde{X}_k) + \lambda_{k+1}(\psi_{S_k}(X_k) - \psi_{S_k}(\widetilde{X}_k)) + Z^{(1)}_{k+1} + V_k(\phi_{S_k} - \psi_{S_k})(X_k) + Z^{(2)}_{k+1}
\end{align*}
where $Z^{(1)}_{k+1}, Z^{(2)}_{k+1} \sim \cN(0, \frac{\eta^2 \sigma^2}{2}I_d)$\footnote{As before, setting $Z^{(1)}_{k+1} \sim \cN(0, \frac{\eta^2 \sigma^2}{\alpha^2}I_d)$ and $Z^{(2)}_{k+1} \sim \cN(0, \frac{\eta^2 \sigma^2}{\beta^2}I_d)$ with $1/\alpha^2 + 1/\beta^2 = 1$ replaces $G(\frac{\sqrt{2}D}{(t-\tau)\eta \sigma}) \otimes C_{b/n}(G(\frac{2\sqrt{2}L}{b \sigma}))$ with $G(\frac{\alpha D}{(t-\tau)\eta \sigma}) \otimes C_{b/n}(G(\frac{2 \beta L}{b \sigma}))$.} are independent and similarly
\begin{align*}
    &\;\psi_{S'_k}(X'_k) + V'_k(\phi'_{S'_k} - \psi_{S'_k})(X'_k) + Z'_{k+1} = \psi_{S'_k}(X'_k) + Z^{(1)'}_{k+1} + V'_k(\phi'_{S'_k} - \psi_{S'_k})(X'_k) + Z^{(2)'}_{k+1}\,.
\end{align*}
Thus
\begin{align*}
    T(\widetilde{X}_{k+1}, X'_{k+1}) &\geq T((\widetilde{X}_k, S_k, \lambda_{k+1}(\psi_{S_k}(X_k) - \psi_{S_k}(\widetilde{X}_k)) + Z^{(1)}_{k+1}), (X'_k, S'_k, Z^{(1)'}_{k+1})) \otimes R(\frac{\eta L}{b}, \frac{\eta \sigma}{\sqrt{2}}, b/n) \\
    &\geq T((\widetilde{X}_k, S_k), (X'_k, S'_k)) \otimes R(a_{k+1}, \frac{\eta \sigma}{\sqrt{2}}) \otimes R(\frac{\eta L}{b}, \frac{\eta \sigma}{\sqrt{2}}, b/n) \\
    &\geq T(\widetilde{X}_k, X'_k) \otimes G(\frac{\sqrt{2}D}{(t-\tau)\eta \sigma}) \otimes C_{b/n}(G(\frac{2\sqrt{2}L}{b \sigma}))\,.
\end{align*}
Repeating this for $k = t-1, \dots, \tau$ yields the result.
\end{proof}

\subsubsection{Choice of $t- \tau$ based on approximation}\label{subsec:cltopt}
Since \autoref{thm:sgd-sc} and \autoref{thm:sgd-proj} hold for every $t - \tau$, we can calculate the corresponding $f$-DP bound for each $t - \tau$ and then take the pointwise maximum as a valid privacy guarantee. However, this may be computationally burdensome if $t$ is large. One way to bypass this calculation is to approximate the composition of subsampled Gaussian mechanisms via CLT (\autoref{lem:clt}), where the resulting $f$-DP bound becomes a GDP bound and thus optimization over $t - \tau$ is analytically tractable. 

\begin{proposition}\label{prop:sgd-sc-approx}
    In the setting of \autoref{thm:sgd-sc}, by choosing (modulo floor or ceiling)
    \[t - \tau = -\frac{\log \frac{b^2 \sigma (1-c)}{2\sqrt{2} nL \sqrt{\log(1/c)}}\sqrt{e^{4L^2/(b \sigma)^2}\Phi(\frac{3L}{b\sigma}) + 3 \Phi(-\frac{L}{b\sigma} )- 2}}{\log(1/c)} - 1
    \]
    $\nsgd$ is approximately $\mu$-GDP, where
    \begin{equation}\label{eq:sgd-sc-approx}
        \mu = \sqrt{8\left(\frac{L}{b\sigma} \frac{c^{t-\tau + 1}}{1-c}\right)^2 + \frac{2b^2}{n^2}(t - \tau)(e^{4L^2/(b \sigma)^2}\Phi(\frac{3L}{b\sigma}) + 3 \Phi(-\frac{L}{b\sigma} )- 2)}\,.
    \end{equation}
\end{proposition}
\begin{proof}
By \autoref{lem:clt}, 
\[C_{b/n}(G(\frac{2\sqrt{2} L}{b \sigma})) \otimes C_{b/n}(G(\frac{2L}{b \sigma}))^{\otimes (t - \tau)} \approx G\left(\sqrt{2}\frac{b}{n}\sqrt{(t - \tau)(e^{4L^2/(b \sigma)^2}\Phi(\frac{3L}{b\sigma}) + 3 \Phi(-\frac{L}{b\sigma} )- 2)}\right)\,.
\]
Also, by bounding 
\[\frac{2\sqrt{2}L}{b \sigma}\frac{c^{t - \tau+1} - c^t}{1-c} \leq \frac{2\sqrt{2}L}{b \sigma}\frac{c^{t - \tau+1}}{1-c}
\]
we obtain an approximate lower bound $G(\mu)$ of the form \eqref{eq:sgd-sc-approx}. As a function of $t - \tau \in (0, \infty)$ it is convex, and the first-order optimality condition provides the stated formula for $t - \tau$.
\end{proof}
\begin{proposition}\label{prop:sgd-proj-approx}
    In the setting of \autoref{thm:sgd-proj}, by choosing (modulo floor or ceiling)
    \[t - \tau = \frac{Dn}{b\eta \sigma \sqrt{e^{8L^2/(b\sigma)^2}\Phi(\frac{3\sqrt{2}L}{b\sigma}) + 3\Phi(-\frac{\sqrt{2}L}{b\sigma}) - 2}}    
    \]
    $\nsgd$ is approximately $\mu$-GDP, where
    \begin{equation}\label{eq:sgd-proj-approx}
        \mu = \sqrt{\frac{2D^2}{\eta^2 \sigma^2 (t - \tau)} + 2\frac{b^2}{n^2}(t-\tau)(e^{8L^2/(b\sigma)^2}\Phi(\frac{3\sqrt{2}L}{b\sigma}) + 3\Phi(-\frac{\sqrt{2}L}{b\sigma}) - 2)}\,.
    \end{equation}
\end{proposition}
\begin{proof}
As in the proof of \autoref{prop:sgd-sc-approx}, the CLT approximation of $C_{b/n}(G(\frac{2\sqrt{2}L}{b \sigma}))^{\otimes (t- \tau)}$ provides a lower bound $G(\mu)$ of the form \eqref{eq:sgd-proj-approx}, which is a convex function in $t - \tau$; the first-order optimality condition yields the stated result.
\end{proof}

\subsection{Lower bounds}
Here we elaborate on lower bounds for the privacy loss (i.e., upper bounds on the $f$-DP guarantee) that complement our results in \autoref{sec:bounds}. Note that an exactly matching bound for $\ngd$ in the strongly convex setting was obtained in \autoref{thm:gd-sc}, and an asymptotically matching bound for $\nsgd$ in the constrained convex setting was obtained in \cite{at22, abt24}. Below, we present results for the other related settings using similar techniques. For the strongly convex setting, these lower bounds are built based on convex quadratics which yield iterates with explicit Gaussians; and for the constrained convex setting, these are obtained by comparing symmetric and biased (projected) Gaussians. We refer the readers to \cite{at22, abt24} for further discussion about these constructions. 

\begin{theorem}\label{thm:proj-lb}
    Consider the setting of \autoref{thm:gd-proj} or \autoref{thm:mgd-proj}. There exist universal constants $0 < c_0 < 1/5, c_1 > 0$ such that if $\sigma^2 \leq c_0\frac{LD}{\eta n}$ and $\mu = c_1\frac{1}{\sigma}\sqrt{\frac{LD}{\eta n}}$, then
    \begin{itemize}
        \item[(a)] $\ngd$ is not $\mu$-GDP for all $t \geq \frac{Dn}{\eta L} \geq \frac{1}{2}$.
        \item[(b)] $\nmgd$ is not $\mu$-GDP for all $E \geq \frac{Db}{\eta L} \geq \frac{1}{2}$.
    \end{itemize} 
\end{theorem}
\begin{proof}
    \begin{itemize}
        \item[(a)] For $\ngd$, let $\mu_0 = \tfrac{\mu}{c_1} = \tfrac{1}{\sigma}\sqrt{\tfrac{LD}{\eta n}}$. Consider $d = 1$ and loss functions such that $\nabla f_i(x) = 0$ for all $i \in [n]$, $\nabla f'_i(x) = 0$ for all $i \neq i^*$ and $\nabla f'_{i^*}(x) = -L$.\footnote{For general $d > 1$, a similar argument (with slightly different constants) can be made by considering $\nabla f'_{i^*}(x) = -Le_1$ and $\cK = [-\Theta(D), \Theta(D)] \times [-\Theta(D/\sqrt{d-1}), \Theta(D/\sqrt{d-1})]^{d-1}$ (constant factors chosen such that $\cK$ has diameter $D$).} Also, let $X_0 = 0$ and $\mathcal{K} = [-\tfrac{D}{2}, \tfrac{D}{2}]$. Note that by \autoref{lem:fdptoapproxdp}, a $\mu$-GDP algorithm is $(\mu^2, \Phi(-\tfrac{\mu}{2}))$-DP. We will show that for $E = [-\tfrac{D}{2}, 0]$,
    \begin{align*}
        \mathbb{P}(X_t \in E) &= \frac{1}{2} \\
        \mathbb{P}(X'_t \in E) &< \exp(-\mu^2)(\frac{1}{2}-\Phi(-\frac{\mu}{2}))
    \end{align*}
    which implies that $\ngd$ is not $(\mu^2, \Phi(-\tfrac{\mu}{2}))$-DP and thus $\ngd$ is not $\mu$-GDP. First, recall that
    \begin{align*}
        X_{k+1} &= \Pi_{\cK}(X_k + Z_{k+1}) \\
        X'_{k+1} &= \Pi_{\cK}(X'_k + \frac{\eta L}{n} + Z'_{k+1})
    \end{align*}
    where $Z_{k+1}, Z'_{k+1} \sim \cN(0, \eta^2 \sigma^2)$. Since the distribution of $X_k$ is symmetric for all $k$, $\mathbb{P}(X_t \in E) = \tfrac{1}{2}$. On the other hand, for $t_0 = t - \lceil 0.8\tfrac{Dn}{\eta L} \rceil + 1$ consider a process $\{X''_k\}_{t_0 \leq k \leq t}$ such that $X''_{t_0} = -\tfrac{D}{2}$ and
    \[X''_{k+1} = \min\{X''_k + \frac{\eta L}{n} + Z'_{k+1}, \frac{D}{2}\}\,.
    \]
    Then inductively, $\mathbb{P}(X'_k \leq z) \leq \mathbb{P}(X''_k \leq z)$ for all $z$. Letting $E_0 = \{\max_{t_0 \leq k \leq t} \sum_{j=t_0}^k Z_j \leq 0.1D\}$, by Doob's submartingale inequality we have
    \[\mathbb{P}(E_0^c) \leq \exp(-\frac{(0.1D)^2}{2 \times \lceil 0.8 \frac{Dn}{\eta L} \rceil \times (\eta \sigma)^2}) \leq \exp(-\frac{0.01 LD}{5.6 n \eta \sigma^2}) = \exp(-\frac{0.01}{5.6}\mu_0^2)\,.
    \]
    Also, conditioning on $E_0$, $X''_t = -\tfrac{D}{2} + \tfrac{\eta L}{n} \times \lceil 0.8 \tfrac{Dn}{\eta L} \rceil + \sum_{j=t_0}^{t} Z_j \geq 0.3D + \sum_{j = t_0}^t Z_j$. Thus
    \begin{align*}
        \mathbb{P}(X'_t \notin E) &\geq \mathbb{P}(X''_{t} > 0) \\
        &\geq \mathbb{P}(\{X''_{t} > 0\} \cap E_0) \\
        &\geq \mathbb{P}(\{0.3D + \sum_{j=t_0}^t Z_j > 0\} \cap E_0) \\
        &\geq \mathbb{P}(0.3D + \sum_{j=t_0}^t Z_j > 0) - \exp(-\frac{0.01}{5.6}\mu_0^2) \\
        &\geq \Phi(\frac{0.3}{\sqrt{2.8}}\mu_0) - \exp(-\frac{0.01}{5.6}\mu_0^2) \\
        &\geq 1 - \exp(-\frac{0.9}{5.6}\mu_0^2) - \exp(-\frac{0.01}{5.6}\mu_0^2)
    \end{align*}
    where the penultimate inequality is from that $\sum_{j=t_0}^t Z_j$ is a mean zero Gaussian with variance $\lceil 0.8 \tfrac{Dn}{\eta L} \rceil \eta^2 \sigma^2 \leq \tfrac{2.8Dn\eta \sigma^2}{L} = \tfrac{2.8D^2}{\mu_0^2}$, and the last inequality is from $\Phi(x) \geq 1 - \exp(-\tfrac{1}{2}x^2)$ for all $x \geq \tfrac{1}{\sqrt{2\pi}}$ (with $0.3\mu_0 / \sqrt{2.8} \geq 0.3 /\sqrt{2.8 c_0} \geq 1/\sqrt{2\pi}$). By taking sufficiently small $c_1 < \sqrt{\tfrac{0.01}{5.6}}$ and $c_0 < c_1^2$ such that 
    \[\exp(-\frac{0.9}{5.6}\mu_0^2) + \exp(-\frac{0.01}{5.6}\mu_0^2) \leq \exp(-c_1^2 \mu_0^2) (\frac{1}{2} - \Phi(-\frac{1}{2}))
    \] 
    for all $\mu_0^2 \geq 1/c_0$, we have
    \begin{align*}
    \exp(-\mu^2)(\frac{1}{2}-\Phi(-\frac{\mu}{2})) &= \exp(-c_1^2 \mu_0^2)(\frac{1}{2}-\Phi(-\frac{c_1\mu_0}{2})) \\
    &\geq \exp(-c_1^2 \mu_0^2)(\frac{1}{2}-\Phi(-\frac{c_1}{2\sqrt{c_0}})) \\
    &> \exp(-c_1^2 \mu_0^2) (\frac{1}{2} - \Phi(-\frac{1}{2})) \\
    &\geq \exp(-\frac{0.9}{5.6}\mu_0^2) + \exp(-\frac{0.01}{5.6}\mu_0^2) \\
    &\geq \mathbb{P}(X'_t \in E)
    \end{align*}
    as desired.
        \item[(b)]     The proof for $\nmgd$ is similar to that for $\ngd$ (recall that $t = lE$ and $n = lb$); consider the same loss functions, initialization, constraint set with $i^* \in B_l$. Then
    \begin{align*}
        X_{k+1} &= \Pi_{\cK}(X_k + Z_{k+1}) \\
        X'_{k+1} &= \Pi_{\cK}(X'_k + \frac{\eta L}{b}\bm{1}_{\{r(k) = l\}} + Z'_{k+1})
    \end{align*}
    where $r(k) = k+1 - l \lfloor \tfrac{k}{l} \rfloor$. For $t_0 = l(E - \lceil 0.8\tfrac{Db}{\eta L} \rceil) + 1$, consider a process $\{X''_k\}_{t_0 \leq k \leq t}$ such that $X''_{t_0} = -\tfrac{D}{2}$ and
    \[X''_{k+1} = \min\{X''_k + \frac{\eta L}{b}\bm{1}_{\{r(k) = l\}} + Z'_{k+1}, \frac{D}{2}\}\,.
    \]
    Then with the same events $E$ and $E_0$, $\mathbb{P}(X_t \in E) = 1/2$ and 
    \begin{align*}
        \mathbb{P}(X'_k \leq z) &\leq \mathbb{P}(X''_k \leq z) \text{ for all } z \\
        \mathbb{P}(E_0^c) &\leq \exp(-\frac{(0.1D)^2}{2 \times l\lceil 0.8 \frac{Db}{\eta L} \rceil \times (\eta \sigma)^2}) \leq \exp(-\frac{0.01}{5.6}\mu_0^2)
    \end{align*}
    and conditioning on $E_0$, $X''_t = -\tfrac{D}{2} + \tfrac{\eta L}{b} \times \lceil 0.8\tfrac{Db}{\eta L} \rceil \geq 0.3D + \sum_{j=t_0}^t Z_j$; the rest are identical.
    \end{itemize}
\end{proof}

\begin{theorem}
    In the setting of \autoref{thm:mgd-sc}, any valid $f$-DP lower bound for $\nmgd$ satisfies
    \[G(\mu) \geq f
    \]
    where $\mu = \frac{L}{b\sigma}\sqrt{\frac{1-c^{lE}}{1 + c^{lE}}\frac{1-c^2}{(1-c^l)^2}}$.
\end{theorem}
\begin{proof}
    Consider the loss functions in the proof of \autoref{thm:gd-sc}, with $X_0 = 0, \mathcal{K} = \R^d$ and $i^* \in B_l$. By direct calculation $X_{lE} = \cN(0, \frac{1-c^{2lE}}{1-c^2} \eta^2 \sigma^2 I_d)$ and $X'_{lE} = \cN(\frac{\eta L}{b} \frac{1-c^{lE}}{1-c^l}v, \frac{1-c^{2lE}}{1-c^2} \eta^2 \sigma^2 I_d)$, implying $T(X_{lE}, X'_{lE}) = G(\mu)$ with $\mu$ as stated.
\end{proof}

\section{Numerical details and results}\label{app:num}
In this section, we provide numerical details of the figures and experiments in the main text and additional numerical results for different algorithms.  Code reproducing these numerics can be found here: \url{https://github.com/jinhobok/shifted_interpolation_dp}.

\subsection{Details for Figure~\ref{fig:fdp-opt}}

In \autoref{fig:fdp-opt}, we consider 1-strongly convex and 10-smooth loss functions with learning rate $\eta = 0.05$, effective sensitivity $L / (n\sigma) = 0.1$, and $t \in \{10, 20, 40, 80, 160\}$, with $t = 160$ in the left figure and $\delta = 10^{-5}$ in the right figure. Our $f$-DP bound is from \autoref{thm:gd-sc}, our RDP bound is from \autoref{thm:gd-sc} and \autoref{lem:fdptordp}, the prior RDP bound is from \cite[Theorem D.6]{ys22}, and the composition bound is from \autoref{thm:gd}. For conversion from GDP and RDP to $(\ep, \delta)$-DP, see \autoref{app:rdp} and \autoref{subsec:tradeoff}. We emphasize that different choices of parameters lead to qualitatively similar plots; see~\autoref{subsec:addnum} for further numerical comparisons in other settings.

\subsection{Details for \S\ref{sec:example}}\label{subsec:example}
Here we provide further numerical details for the experiment in~\autoref{sec:example}.
The purpose of this simple numerical example is to corroborate our theoretical findings by comparing them with existing privacy bounds. As such, we simply compare algorithms with the same hyperparameters, and do not attempt to optimize these choices for individual algorithms. 

In~\autoref{sec:example},~\autoref{tab:lr-ep} shows that our results provide improved privacy bounds. That table considers the privacy leakage of $\nmgd$ in $(\eps,\delta)$-DP with regularization parameter $\lambda = 0.002$. \autoref{tab:lr-mu-long} and \autoref{tab:lr-ep-long} provide more details on this numerical comparison by also considering another algorithm ($\nsgd$), another notion of privacy leakage (GDP), and another parameter ($\lambda = 0.004$). Details on these tables: for the GDP Composition privacy bound on $\nsgd$, we present the approximate value of the GDP parameter provided by CLT since this is computationally tractable; for $(\ep, \delta)$-DP we compute the corresponding $\ep$ to an error of $10^{-3}$ using the numerical procedure in \autoref{app:convsubgdp}; and we convert the currently known best RDP bounds provided by \cite[Theorem 3.3]{ys22} to $(\ep, \delta)$-DP using the numerical procedure in \autoref{app:rdp}.

\begin{table}[H]
\centering
\caption{
   More detailed version of~\autoref{tab:lr-ep}, for GDP. Lists the GDP parameters of private algorithms for the regularized logistic regression problem.  Note that GDP Composition yields the same privacy bound regardless of the regularization parameter. Our results provide improved privacy.
}
\label{tab:lr-mu-long}
\vskip 0.15in
\begin{tabular}{c|cc|cc}
\toprule
Epochs & \multicolumn{2}{c|}{GDP Composition} & \multicolumn{2}{c}{Our Bounds} \\
 \midrule
Algorithms & $\nmgd$ & $\nsgd$ & \multicolumn{2}{c}{$\nmgd$}\\
$\lambda$ & \multicolumn{2}{c|}{$\{0.002, 0.004\}$} & 0.002 & 0.004 \\
\midrule
50 & 4.71 & 1.03 & 0.99 & 0.99 \\
100 & 6.67 & 1.45 & 1.24 & 1.22 \\
200 & 9.43 & 2.05 & 1.59 & 1.51 \\
\bottomrule
\end{tabular}
\vskip -0.1in
\end{table}

\begin{table}[H]
\centering
\caption{More detailed version of~\autoref{tab:lr-ep}, for $(\eps,\del)$-DP. Lists $\ep$ of private algorithms on the regularized logistic regression problem for $\delta = 10^{-5}$. Note that GDP Composition yields the same privacy bound regardless of $\lambda$. Our results provide improved privacy over both GDP Composition and RDP.}
\label{tab:lr-ep-long}
\vskip 0.15in
\begin{tabular}{c|cc|cc|cc}
\toprule
Epochs & \multicolumn{2}{c|}{GDP Composition} & \multicolumn{2}{c|}{RDP} & \multicolumn{2}{c}{Our Bounds} \\
 \midrule
Algorithms & \multicolumn{1}{c}{$\nmgd$} & \multicolumn{1}{c|}{$\nsgd$} & \multicolumn{2}{c|}{$\nmgd$} & \multicolumn{2}{c}{$\nmgd$} \\
$\lambda$ & \multicolumn{2}{c|}{$\{0.002, 0.004\}$} & 0.002 & 0.004 & 0.002 & 0.004 \\
\midrule
50 & 30.51 & 4.44 & 5.82 & 5.61 & 4.34 & 4.32 \\
100 & 49.88 & 6.65 & 7.61 & 7.00 & 5.60 & 5.51 \\
200 & 83.83 & 10.11 & 9.88 & 8.38 & 7.58 & 7.09 \\
\bottomrule
\end{tabular}
\vskip -0.1in
\end{table}

These tables show that compared to our results (\autoref{thm:mgd-sc}), the standard GDP Composition bound for $\nmgd$ (\autoref{thm:mgd}) provides essentially no privacy. This is because that standard bound incurs a large privacy loss in each epoch (at the step in which the adjacent datasets use different gradients), and this privacy leakage accumulates indefinitely---whereas our analysis captures the contractivity of the algorithm's updates, which effectively ensures that previous gradient queries leak less privacy the longer ago they were performed. 
See \autoref{sec:tech} for a further discussion of this. Combined with the lossless conversion enabled by our $f$-DP analysis, our results also provide better privacy than the state-of-the-art RDP bounds.

\autoref{tab:lr-train} and \autoref{tab:lr-test} (reporting (mean) $\pm$ (standard deviation) of accuracies over 10 runs) show that (1) $\ncgd$ and $\nsgd$ have comparable training and test accuracy for this problem, and (2) both algorithms improve when run longer, thus necessitating better privacy guarantees in order to achieve a target error (for either training or test) given a fixed privacy budget. Note that while $\nsgd$ enjoys better privacy bounds than $\ncgd$ using the standard GDP Composition argument, our new privacy guarantees for $\nmgd$ improve over GDP Composition bounds for both algorithms (c.f., \autoref{tab:lr-mu-long} and \autoref{tab:lr-ep-long}). In particular, observe that while running algorithms longer leads to better accuracy, the privacy leak in $\nsgd$ from GDP Composition grows faster relative to our results (e.g., compare the values of $\ep$ when $E = 50$ and $E = 200$). This highlights the convergent dynamics of our privacy bounds and exemplifies how this enables algorithms to be run longer while preserving privacy.

\begin{table}[H]
\centering
\caption{More detailed version of~\autoref{tab:lr-acc}. Lists \emph{training} accuracy (\%) of $\nmgd$ and $\nsgd$ for regularized logistic regression. Note that both algorithms perform similarly and improve when run longer.}
\label{tab:lr-train}
\vskip 0.15in
\begin{tabular}{c|cc|cc}
\toprule
Epochs & \multicolumn{2}{c|}{$\nmgd$} & \multicolumn{2}{c}{$\nsgd$} \\ \midrule
$\lambda$ & 0.002 & 0.004 & 0.002 & 0.004 \\
\midrule
50 & 89.36 $\pm$ 0.03 & 89.23 $\pm$ 0.02 & 89.36 $\pm$ 0.04 & 89.22 $\pm$ 0.04  \\
100 & 90.24 $\pm$ 0.03 & 90.00 $\pm$ 0.03 & 90.25 $\pm$ 0.02 & 89.99 $\pm$ 0.03  \\
200 & 90.85 $\pm$ 0.02 & 90.39 $\pm$ 0.04 & 90.84 $\pm$ 0.03 & 90.37 $\pm$ 0.02  \\
\bottomrule
\end{tabular}
\vskip -0.1in
\end{table}

\begin{table}[H]
\centering
\caption{More detailed version of~\autoref{tab:lr-acc}. Lists \emph{test} accuracy (\%) of $\nmgd$ and $\nsgd$ for regularized logistic regression. Again, note that both algorithms perform similarly and improve when run longer.}
\label{tab:lr-test}
\vskip 0.15in
\begin{tabular}{c|cc|cc}
\toprule
Epochs & \multicolumn{2}{c|}{$\nmgd$} & \multicolumn{2}{c}{$\nsgd$} \\ \midrule
$\lambda$ & 0.002 & 0.004 & 0.002 & 0.004 \\
\midrule
50 & 90.12 $\pm$ 0.04 & 90.03 $\pm$ 0.07 & 90.12 $\pm$ 0.08 & 90.00 $\pm$ 0.06  \\
100 & 90.94 $\pm$ 0.07 & 90.70 $\pm$ 0.05 & 90.97 $\pm$ 0.04 & 90.75 $\pm$ 0.03  \\
200 & 91.37 $\pm$ 0.08 & 91.02 $\pm$ 0.07 & 91.40 $\pm$ 0.07 & 91.01 $\pm$ 0.04  \\
\bottomrule
\end{tabular}
\vskip -0.1in
\end{table}

For the experiment, we closely follow the setting considered in \cite{ys22}---for proofs and details on theoretical guarantees with respect to the setting, see \cite[Section 5]{ys22}. The MNIST dataset has $n = 60000$ training data points and 10000 test data points; for both $\nmgd$ and $\nsgd$, we set the parameters as $C = 8$, $\eta = 0.05$, $b = 1500$, $\sigma = 1/100$, $L = 10$, $E \in \{50, 100, 200\}$ and $\lambda \in \{0.002, 0.004\}$. 
First, we clip the feature so that it has norm $C$. For the loss function $l(\theta, (x, y))$ of the (unregularized) logistic regression, we calculate the gradient for each data point $(x, y)$ as
\[\nabla f(\theta, (x, y)) = \frac{\nabla l(\theta, (x, y))}{\norm{\nabla l(\theta, (x, y))}} \cdot \min\{\norm{\nabla l(\theta, (x, y))}, \frac{L}{2}\} + \lambda \theta\,.
\]
In other words, we first clip the gradient by $L/2$ so that the gradient sensitivity is $L$, and add a gradient $\lambda \theta$ of the regularization term $(\lambda/2)\norm{\theta}^2$ (which does not affect the gradient sensitivity).

\subsection{Additional numerics}\label{subsec:addnum}

Here we provide additional numerical results to illustrate our privacy bounds in \autoref{sec:bounds}, by comparing our $f$-DP bounds with the counterparts derived by the standard GDP Composition analysis. We cover the settings and algorithms covered in the main text over a broad range of parameters, emphasizing the convergent dynamics of our privacy bounds. The different settings lead to qualitatively similar comparisons. Recall that the relevant parameters of the algorithms are the learning rate $\eta$, noise rate $\sigma$, number of data points $n$, batch size $b$, gradient sensitivity $L$, and diameter $D$ of the constraint set $\cK$; see \autoref{ssec:prelim:algs}.

\subsubsection{$\ngd$}\autoref{fig:ngd-fdp-approxdp-sc} shows our results for $f$-DP (left) and its conversion into $(\ep, \delta)$-DP (right) for $\ngd$ in the strongly convex setting (\autoref{thm:gd-sc}), where our bound is exact. In contrast, observe that while the bound from GDP Composition is nearly tight for a small number of iterations $t$, the guarantee becomes vacuous as $t$ increases. This is also evident from the $(\ep, \delta)$-DP plot, where the discrepancy between the two bounds increases in $t$.

In \autoref{tab:ngd-fdp-sc}, since we obtain GDP bounds, we provide the GDP parameter $\mu$ as a function of the number of iterations $t$ and the contractivity  $c = \max\{|1 - \eta m|, |1 - \eta M|\}$. All values in \autoref{tab:ngd-fdp-sc} scale linearly in the effective sensitivity $L/(n\sigma)$; for simplicity we set it to $0.1$. Note that the GDP Composition bound is independent of $c$ because it is not ``geometrically aware'' in the sense described in~\autoref{sec:tech}. Our bound is optimal and always improves over GDP composition---substantially so as $t$ increases.

\begin{figure}[H]
    \centering
    \includegraphics[width=0.9\textwidth]{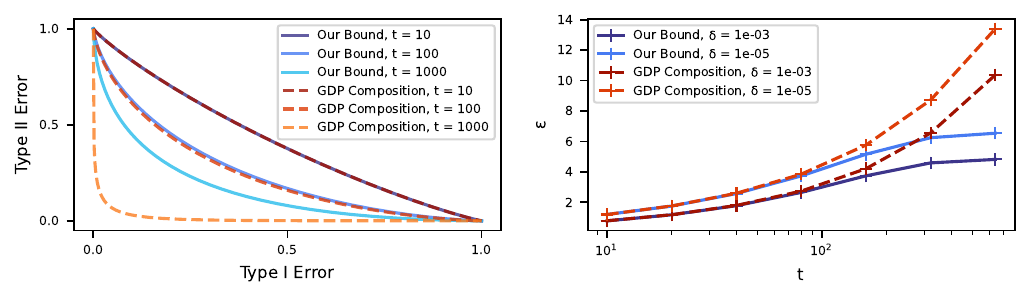}
    \caption{Comparison of our exact privacy characterization (\autoref{thm:gd-sc}) with the standard GDP Composition bound (\autoref{thm:gd}) for $\ngd$, for $c = 0.99$. Shown for $f$-DP (left) and $(\ep, \delta)$-DP (right).}
    \label{fig:ngd-fdp-approxdp-sc}
\end{figure}

\begin{table}[H]
\centering
\caption{GDP parameter $\mu$ from our exact privacy characterization (\autoref{thm:gd-sc}), for varying $t$ and $c$.}

\label{tab:ngd-fdp-sc}
\vskip 0.15in
\begin{tabular}{c|c|ccccc}
\toprule
Steps & \multicolumn{1}{c|}{GDP Composition} & \multicolumn{5}{c}{Our Bounds} \\
\midrule
$c$ & $\{0.92, 0.96, 0.98, 0.99, 0.995\}$ & 0.92 & 0.96 & 0.98 & 0.99 & 0.995 \\
\midrule
10 & 0.316  & 0.308 & 0.314 & 0.316 & 0.316 & 0.316\\
100 & 1.000  & 0.490 & 0.688 & 0.871 & 0.961 & 0.990\\
1000 & 3.162 & 0.490 & 0.700 & 0.995 & 1.411 & 1.984 \\
\bottomrule
\end{tabular}
\vskip -0.1in
\end{table}

\autoref{fig:ngd-fdp-approxdp-proj} and \autoref{tab:ngd-fdp-proj} turn to the setting of constrained convex optimization in \autoref{thm:gd-proj}. In the $(\ep, \delta)$-DP figure, we plot the minimum $\ep$ between \autoref{thm:gd-proj} and GDP Composition. A distinctive feature from both plots is that our privacy bound stays constant after a number of iterations, compared to GDP Composition. In particular, there exists a threshold $t^* = t^*(L/n, \eta)$ such that the algorithm can run beyond the threshold (and even indefinitely) with a provable guarantee of $\mu^*$-GDP. To highlight this fact, we provide the pairs of $(t^*, \mu^*)$ in the table over multiple combinations of parameters. We set the diameter of the constraint set $\cK$ to be $D = 1$ and noise parameter to be $\sigma = 8$; note that as in the previous case, the GDP parameters in this setting scale linearly with respect to $1/\sigma$. Other parameter choices lead to qualitatively similar comparisons. 

\begin{figure}[H]
    \centering
    \includegraphics[width=0.9\textwidth]{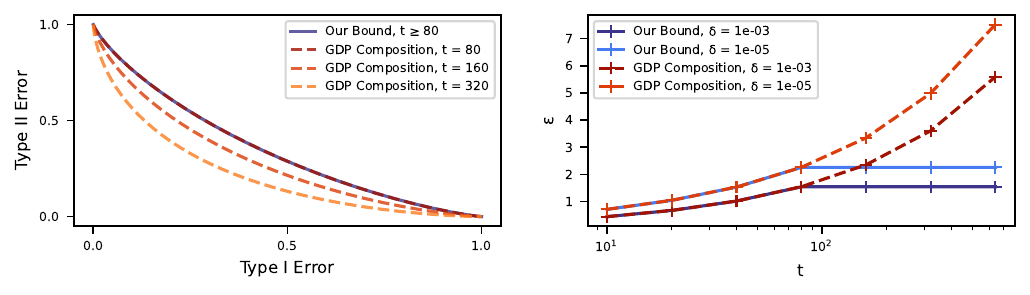}
    \caption{Comparison of our bound (\autoref{thm:gd-proj}) with the standard GDP Composition bound (\autoref{thm:gd}) for $\ngd$, for $L/n = 0.5$ and $\eta = 0.1$. Shown for $f$-DP (left) and $(\ep, \delta)$-DP (right).}
    \label{fig:ngd-fdp-approxdp-proj}
\end{figure}

\begin{table}[H]
\centering
\caption{Threshold number of iterations $t^*$ at which point our GDP bound $\mu^*$ from~\autoref{thm:gd-proj} no longer increases. Shown for varying parameters $L/n$ and $\eta$.}
\label{tab:ngd-fdp-proj}
\vskip 0.15in
\begin{tabular}{c|ccc}
\toprule
$L/n \setminus \eta$ & 0.2 & 0.1 & 0.05 \\
\midrule
0.25 & (80, 0.280) & (160, 0.395) & (320, 0.559) \\
0.5 & (40, 0.395) & (80, 0.559) & (160, 0.791) \\
1 & (20, 0.559) & (40, 0.791) & (80, 1.118) \\
\bottomrule
\end{tabular}
\vskip -0.1in
\end{table}

\subsubsection{$\nmgd$} \autoref{fig:ncgd-fdp-approxdp-sc} and \autoref{tab:ncgd-fdp-sc} show the analog of \autoref{fig:ngd-fdp-approxdp-sc} and \autoref{tab:ngd-fdp-sc}, now for $\nmgd$ rather than $\ngd$. Recall that $l$ denotes the number of batches and $c = \max\{|1 - \eta m|, |1 - \eta M|\}$ is the contraction factor. All GDP parameters scale linearly in the effective sensitivity $L/(b\sigma)$; we set it to $0.2$ for concreteness. The improvement of our bounds over the standard GDP Composition bound is pronounced: our bounds yield strong privacy in both $f$-DP (left) and $(\ep, \delta)$-DP (right), whereas the GDP Composition bound becomes effectively non-private as the number of epochs $E$ increases. This is also evident from \autoref{tab:ncgd-fdp-sc}, where our bound produces better privacy (even with $E=500$ epochs) than GDP Composition (even with $E = 5$). 

\begin{figure}[H]
    \centering
    \includegraphics[width=0.9\textwidth]{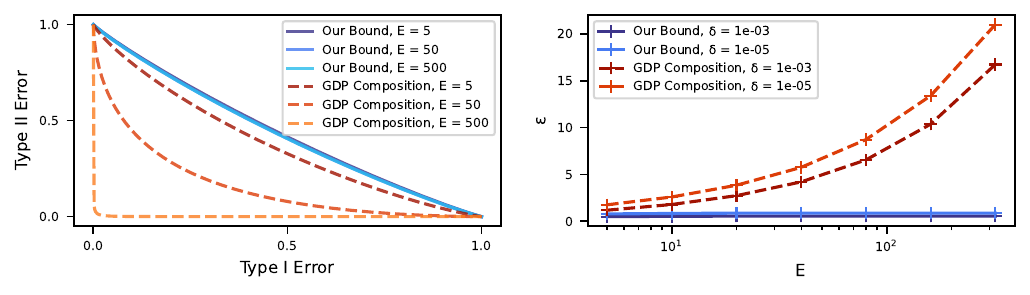}
    \caption{(Left) $f$-DP, (Right) $(\ep, \delta)$-DP comparison of our bound (\autoref{thm:mgd-sc}) with the standard GDP Composition bound (\autoref{thm:mgd}) for $\nmgd$, for $c = 0.99$ and $l = 20$.}
    \label{fig:ncgd-fdp-approxdp-sc}
\end{figure}

\begin{table}[H]
\centering
\caption{GDP parameter $\mu$, for varying number of epochs $E$ and contractivity $c$.}
\label{tab:ncgd-fdp-sc}
\vskip 0.15in
\begin{tabular}{c|c|ccc|ccc|ccc}
\toprule
Epochs & GDP Composition & \multicolumn{9}{c}{Our Bounds} \\
 \midrule
$l$ & \multicolumn{1}{c|}{$\{10, 20, 40\}$} & \multicolumn{3}{c|}{10} & \multicolumn{3}{c|}{20} & \multicolumn{3}{c}{40} \\
$c$ & \multicolumn{1}{c|}{$\{0.98, 0.99, 0.995\}$} & 0.98 & 0.99 & 0.995 & 0.98 & 0.99 & 0.995 & 0.98 & 0.99 & 0.995\\
\midrule
5 & 0.447  & 0.229 & 0.233 & 0.235 & 0.211 & 0.215 & 0.217 & 0.202 & 0.205 & 0.208\\
50 & 1.414 & 0.270 & 0.334 & 0.410 & 0.216 & 0.237 & 0.275 & 0.203 & 0.208 & 0.219 \\
500 & 4.472 & 0.270 & 0.336 & 0.439 & 0.216 & 0.237 & 0.276 & 0.203 & 0.208 & 0.219 \\
\bottomrule
\end{tabular}
\vskip -0.1in
\end{table}

Next we turn to the setting of constrained convex losses from~\autoref{thm:mgd-proj}. Again, our bounds converge in the number of epochs quickly and uniformly improve over the bounds from GDP Composition after only a few number of epochs---for the $(\ep, \delta)$-DP plot, we show the minimum $\ep$ between \autoref{thm:mgd-proj} and GDP Composition. In particular, from the table one can observe that there are even a few cases in which our bounds are better than those from GDP Composition after less than 10 epochs. We chose $D = 1$ and $\sigma = 3$; the GDP parameters scale linear in $1/\sigma$.

\begin{figure}[H]
    \centering
    \includegraphics[width=0.9\textwidth]{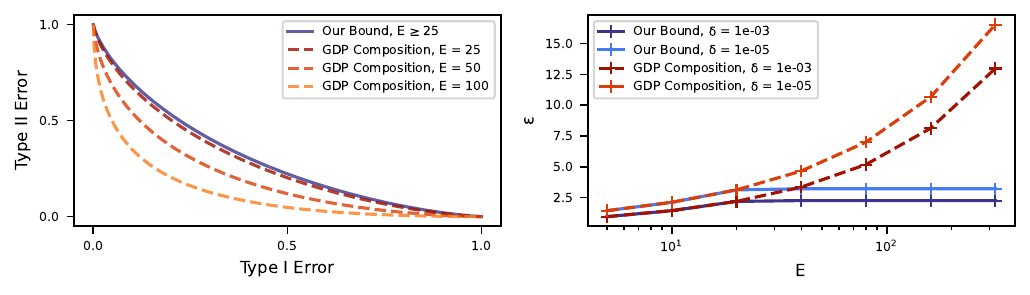}
    \caption{(Left) $f$-DP, (Right) $(\ep, \delta)$-DP comparison of our bound (\autoref{thm:mgd-proj}) with the existing bound from GDP Composition (\autoref{thm:mgd}) for $\nmgd$ under constrained set, with $\eta = 0.02, L/b = 0.5$ and $l = 20$.}
    \label{fig:ncgd-fdp-approxdp-proj}
\end{figure}

\begin{table}[H]
\centering
\caption{$(E^*, \mu^*)$ over different values of $(L/b, \eta)$, with $l = 10$.}
\label{tab:ngd-fdp-proj-10}
\vskip 0.15in
\begin{tabular}{c|ccc}
\toprule
$L/b \setminus \eta$ & 0.04 & 0.02 & 0.01 \\
\midrule
0.25 & (31, 0.534) & (62, 0.750) & (123, 1.057) \\
0.5 & (17, 0.764) & (33, 1.067) & (65, 1.500) \\
1 & (10, 1.106) & (20, 1.528) & (40, 2.134) \\
\bottomrule
\end{tabular}
\vskip -0.1in
\end{table}

\begin{table}[H]
\centering
\caption{$(E^*, \mu^*)$ over different values of $(L/b, \eta)$, with $l = 20$.}
\label{tab:ngd-fdp-proj-20}
\vskip 0.15in
\begin{tabular}{c|ccc}
\toprule
$L/b \setminus \eta$ & 0.04 & 0.02 & 0.01 \\
\midrule
0.25 & (16, 0.382) & (31, 0.534) & (62, 0.750) \\
0.5 & (9, 0.553) & (17, 0.764) & (33, 1.067) \\
1 & (5, 0.816) & (10, 1.106) & (20, 1.528) \\
\bottomrule
\end{tabular}
\vskip -0.1in
\end{table}

\begin{table}[H]
\centering
\caption{$(E^*, \mu^*)$ over different values of $(L/b, \eta)$, with $l = 40$.}
\label{tab:ngd-fdp-proj-40}
\vskip 0.15in
\begin{tabular}{c|ccc}
\toprule
$L/b \setminus \eta$ & 0.04 & 0.02 & 0.01 \\
\midrule
0.25 & (8, 0.276) & (16, 0.382) & (31, 0.534) \\
0.5 & (5, 0.408) & (9, 0.553) & (17, 0.764) \\
1 & (3, 0.624) & (5, 0.816) & (10, 1.106) \\
\bottomrule
\end{tabular}
\vskip -0.1in
\end{table}

\subsubsection{$\nsgd$} For brevity, here we consider just the setting of constrained convex losses; similar plots can be obtained for the strongly convex setting.~\autoref{fig:sgd-fdp-proj} compares our new privacy bound (\autoref{thm:sgd-proj}) with the standard GDP Composition bound (\autoref{thm:sgd}), by illustrating the $f$-DP tradeoff curves of both bounds for a broad range of parameters. For most parameter choices, our bounds provide reasonable privacy that is valid even beyond the number of iterations in the plots. On the other hand, the divergence of the GDP Composition bound  clearly degrades the privacy as the number of iterations increases.

\begin{figure}[H]
    \centering
    \includegraphics[width=0.9\textwidth]{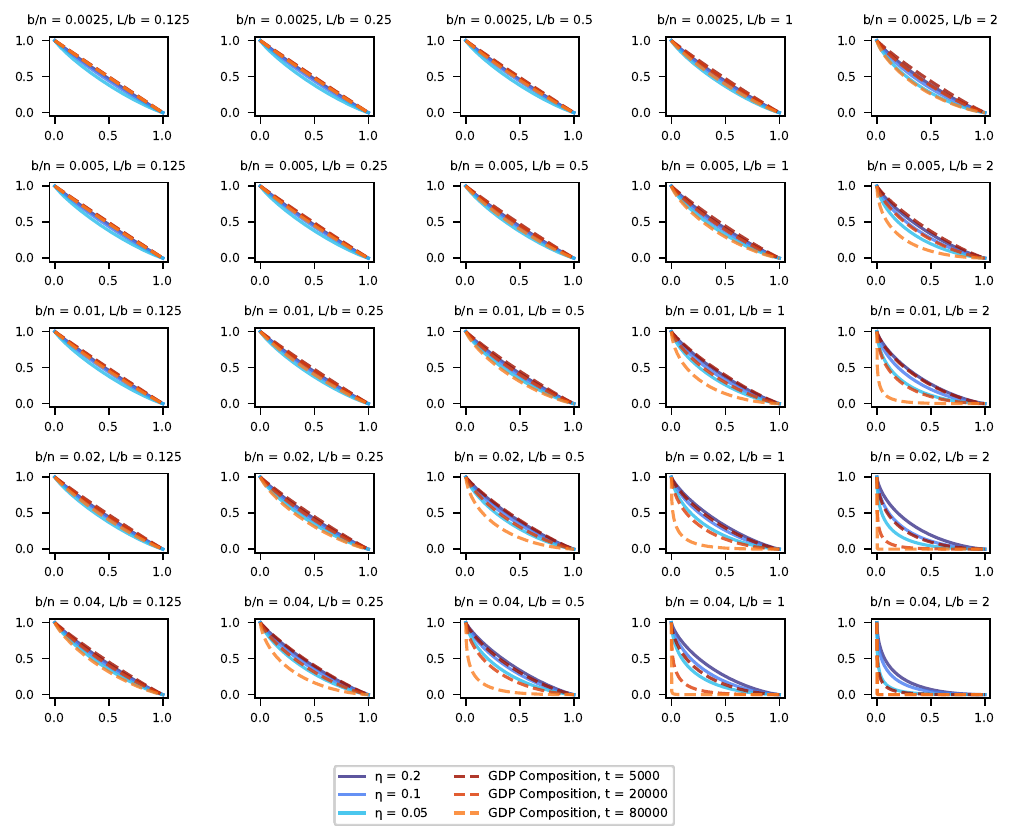}
    \caption{Comparison of our new privacy bound (\autoref{thm:sgd-proj}) with the standard GDP Composition bound (\autoref{thm:sgd}) for $\nsgd$ in the setting of constrained convex losses. Each subplot illustrates the $f$-DP tradeoff curve of these bounds, for a given relative batch size $b/n$ and effective sensitivity $L/b$.}
    \label{fig:sgd-fdp-proj}
\end{figure}

To make this figure, we approximated the compositions of $C_{b/n}(G(\cdot))$ in \autoref{thm:sgd-proj} by CLT, by choosing the best privacy bound among $t - \tau \in \{100, 200, \dots, 4900\}$ after applying CLT; this is a valid approximation by \autoref{lem:complim}. We also note that for improved numerical results, instead of the respective factors of $(\sqrt{2}, \sqrt{2})$ inside $G(\cdot)$ and $C_{b/n}(G(\cdot))$ of the statement of \autoref{thm:sgd-proj} we used $(\sqrt{10}, \sqrt{10}/3)$ (see the proof of \autoref{thm:sgd-proj} for details). For parameters unspecified in the plots we used $D = 1$ and $\sigma = 3$. 

\subsection{Comparison of privacy bounds for the exponential mechanism}\label{subsec:expmechnum}

Let $f_\ep$ be the tradeoff function corresponding to $(\ep, 0)$-DP. From \cite[Proposition 3]{drs22}, it is straightforward to check that $f_\ep \geq G(\mu)$ iff
\[e^\ep \leq \frac{1 - \Phi(-\frac{\mu}{2})}{\Phi(-\frac{\mu}{2})}\,.\]
Plugging in $\ep = 2LD$ and $\mu = 2\sqrt{LD}$ from \autoref{cor:expmechconvex}, one can check (using standard nonlinear equation solvers) that this holds iff $LD \leq c^*$ where $c^* \approx 0.676.$

\subsection{Numerical composition of subsampled GDP}\label{app:convsubgdp}

Here we mention details on the numerical procedure for calculating an $(\ep, \delta)$-DP bound from an $f$-DP bound of the form $f = C_p(G(\mu))^{\otimes t}$; this is used in \autoref{sec:example}. This formula appears in multiple settings of $\nsgd$, with each $C_p(G(\mu))$ representing the $f$-DP of subsampled Gaussian mechanism. This conversion process is important for both notions: First, while the composition can be approximated by CLT (\autoref{lem:clt}), in practice the approximation is not enough to guarantee whether the algorithm achieves a given privacy budget, typically expressed in $(\ep, \delta)$-DP. On the other hand, if one can obtain an accurate collection of different $(\ep, \delta)$-DP bounds, it can be converted into an $f$-DP bound of comparable accuracy due to the duality between the two notions \cite[Proposition 5 \& 6]{drs22}.

We implement the framework of \textit{privacy loss random variables} (PRV)---which is an equivalent notion of $f$-DP---and the corresponding analytical procedure provided in \cite{gopi2021numerical}. Given a fixed value of $\delta$ and (possibly different) compositions of private mechanisms, the algorithm presented in the paper allows one to numerically calculate $\ep$ with user-specified margin of error. We refer the readers to \cite{gopi2021numerical} for the background and overview of the PRV framework and only present relevant results for the problem of our interest.\footnote{We also note that while a corresponding result for the one-sided version of $G(\mu)_p = T(\cN(0, 1), p\cN(\mu, 1) + (1-p)\cN(0, 1))$ was already presented in \cite{gopi2021numerical} and is often interchangeably used, it is quantitatively different from $C_p(G(\mu))$, even in the limiting regime of CLT. Compare, for example, \autoref{lem:clt} and \cite{bdls20}.}

The privacy curve and PRV are characterized as follows \cite[Definition 2.1 \& 3.1]{gopi2021numerical}.

\begin{definition}[Privacy curve and PRV] Let $f = T(X, Y)$ be a tradeoff function. Then the \emph{privacy curve} $\delta: \R \to [0, 1]$ with respect to $(X, Y)$ is defined as $\delta(X||Y)(\ep) = \sup_{E}\{ \mathbb{P}(Y \in E) - e^{\ep} \mathbb{P}(X \in E)\}$ where the supermum is over all events. Conversely, given a privacy curve $\delta: \R \to [0, 1]$, $(X, Y)$ are \emph{privacy loss random variables} if the following holds.
\begin{itemize}
    \item $X, Y$ are supported on the extended real line $\bar{\R}$.
    \item $\delta(X||Y) \equiv \delta$.
    \item Let $X(t), Y(t)$ respectively be the probability density functions of $X, Y$. Then $Y(t) = e^tX(t)$ and $Y(-\infty) = X(\infty) = 0$. 
\end{itemize}
\end{definition}

The probability density functions of PRVs can be calculated from the privacy curve \cite[Theorem 3.3]{gopi2021numerical}.
\begin{lemma}[Conversion]\label{lem:pctoprv} Given a privacy curve $\delta: \R \to [0, 1]$, the probability density functions of its PRVs $(X, Y)$ are given as $Y(t) = \delta''(t) - \delta'(t)$ and $X(t) = e^t(\delta''(t) - \delta'(t))$.
\end{lemma}
Also, symmetric tradeoff functions have the simple form of PRVs $(X, Y)$ with $X = -Y$ \cite[Proposition D.9]{gopi2021numerical}.
\begin{lemma}[Symmetry]\label{lem:prvsymm} If $(X, Y)$ are PRVs for a privacy curve $\delta(P||Q)$, the PRVs for $\delta(Q||P)$ are $(-Y, -X)$. In particular, if the privacy curve is symmetric (i.e., $\delta(P||Q) = \delta(Q||P)$; equivalently, the corresponding tradeoff function is symmetric) then $X = -Y$.
\end{lemma}

By the following result, we can numerically calculate the $(\ep, \delta)$-DP converted from $f$-DP for $f = C_p(G(\mu))^{\otimes t}$.

\begin{proposition}\label{prop:subgdpcdf} Let $(X, Y)$ be such that the CDF of $Y$ is given as
\[F_Y(t) = \begin{cases} p\Phi(\frac{\ep^{+}}{\mu} - \frac{\mu}{2}) + (1-p)\Phi(\frac{\ep^{+}}{\mu} + \frac{\mu}{2}) & t > 0 \\ \Phi(-\frac{\ep^{-}}{\mu} - \frac{\mu}{2}) & t \leq 0
\end{cases}
\]
where $\ep^{+} = \log((p-1+e^t)/p), \ep^{-} = \log((p-1 + e^{-t})/p)$ and $X = -Y$. Then $(X, Y)$ are PRVs for the tradeoff function $C_p(G(\mu))$.
\end{proposition}

\begin{proof}
    Let $\delta, \delta_0$ respectively be the privacy curves of $C_p(G(\mu))$ and $G(\mu)_p$. Then it is straightforward to check that
    \[\delta(t) = \begin{cases}\delta_0(t) & t > 0 \\ 1 - e^t(1-\delta_0(-t)) & t \leq 0 \end{cases}
    \]
    from the definition of $C_p(G(\mu))$ (as a symmetrized version of $G(\mu)_p$; see \autoref{def:cp}) and the duality between $(\ep, \delta)$-DP and $f$-DP. By taking antiderivative from \autoref{lem:pctoprv}, the CDF of $Y$ is given as
    \[F_Y(t) = \begin{cases} \delta_0'(t) - \delta_0(t) + C & t > 0 \\ -e^t \delta_0'(-t) - 1 + C & t \leq 0\end{cases}
    \]
    for some constant $C$. By either obtaining $\delta_0(t)$ directly from \cite[Lemma 2]{drs22} or comparing the $t > 0$ part of $F_Y(t)$ with \cite[Proposition C.4]{gopi2021numerical}, one can derive the formula of $F_Y(t)$ as stated with $C = 1$. Also, $X = -Y$ from \autoref{lem:prvsymm}.
\end{proof}

\end{document}